%% file: main_arXive.tex
\let\ORIlabel\label
\let\ORIrefstepcounter\refstepcounter
\AddToHook{package/hyperref/before}{
	\let\label\ORIlabel 
	\let\refstepcounter\ORIrefstepcounter}
\documentclass{siamonline220329}

\input{ex_shared}

\ifpdf
\hypersetup{
	pdftitle={Deep greedy unfolding: Sorting out argsorting in greedy sparse recovery algorithms},
	pdfauthor={S. Mohammad-Taheri, M. J. Colbrook, and S. Brugiapaglia}
}
\fi

\begin{document}
	\maketitle
	
	\begin{abstract}
		Gradient-based learning imposes (deep) neural networks to be differentiable at all steps. This includes model-based architectures constructed by unrolling iterations of an iterative algorithm onto layers of a neural network, known as \emph{algorithm unrolling}. However, greedy sparse recovery algorithms depend on the non-differentiable argsort operator, which hinders their integration into neural networks. 
		In this paper, we address this challenge in Orthogonal Matching Pursuit (OMP) and Iterative Hard Thresholding (IHT), two popular representative algorithms in this class. We propose permutation-based variants of these algorithms and approximate permutation matrices using ``soft'' permutation matrices derived from \emph{softsort}, a continuous relaxation of argsort. We demonstrate--both theoretically and numerically--that Soft-OMP and Soft-IHT, as differentiable counterparts of OMP and IHT and fully compatible with neural network training, effectively approximate these algorithms with a controllable degree of accuracy.
		This leads to the development of OMP- and IHT-Net, fully trainable network architectures based on Soft-OMP and Soft-IHT, respectively. Finally, by choosing weights as ``structure-aware'' trainable parameters, we connect our approach to structured sparse recovery and demonstrate its ability to extract latent sparsity patterns from data.
		
	\end{abstract}
	
	
	\section{Introduction}
	The last decade has witnessed grand technological developments, ranging from language models to industrial robots, due to Artificial Intelligence (AI). These advancements mostly rely on \emph{deep learning}---a flexible data-driven approach that learns from observables to make predictions on unseen data \cite{goodfellow2016deep}. However, while deep learning continues to push the boundaries of machine intelligence, concerns remain about its safety and trustworthiness in critical sectors such as autonomous vehicles and medical imaging \cite{choi20217,heaven2019deep}. These concerns stem primarily from a general lack of interpretability and stability in such models.
	
	On the other hand, \emph{sparse recovery} algorithms, which are used to obtain a parsimonious representation of the data and backed by proven breakthroughs in modern compression and data acquisition \cite{adcock2021compressive, foucart2013mathematical}, benefit from concrete mathematical justifications, including recovery guarantees and convergence analysis.
	
	Deep learning and sparse recovery algorithms are unified under the umbrella of \emph{algorithm unrolling} (also known as \emph{unfolding} or \emph{unraveling}) \cite{liang2020deep, monga2021algorithm, scarlett2022theoretical}. This paradigm involves unfolding a (sparse recovery) algorithm into the layers of a neural network, where certain parameters of the original algorithm are treated as trainable, thus optimizing the performance of the algorithm with respect to those parameters.
	
	However, training neural networks, which relies on gradient-based optimization, necessitates differentiability at all steps. This requirement clashes with the fact that a broad class of sparse recovery methods, known as \emph{greedy and thresholding algorithms} \cite{foucart2013mathematical}, employs the non-differentiable \emph{argsort} operator---a piece-wise constant operator (see \Cref{fig:softsort}). This operator poses a challenge for unrolled neural networks, rendering them non-trainable due to its ineffective zero or non-existent gradients. In this manuscript we bridge this gap by using \emph{softsort} \cite{prillo2020softsort}, a differentiable relaxation of the argsort operator, and integrating it into the iterations of two representative algorithms from this family: \emph{Orthogonal Matching Pursuit} (OMP) \cite{davis1994adaptive,pati1993orthogonal} and \emph{Iterative Hard Thresholding} (IHT) \cite{blumensath2008iterative,fornasier2008iterative}. The resulting differentiable variants, that we call \emph{Soft-OMP} and \emph{Soft-IHT}, are shown theoretically and experimentally to be valid approximations of their non-differentiable counterparts.
	
	\subsection{Main contributions}
	Our main contributions are summarized as follows:
	\begin{enumerate}
		\item We propose gradient-friendly versions of OMP and IHT by reinterpreting these algorithms through a permutation-based lens, approximating the permuation matrices associated with argsort using differentiable approximants that we derive from the softsort operator. To the best of our knowledge, this is the first time the non-differentiability of argsort-based operators in these algorithms has been directly addressed through a mathematical treatment rather than engineering-based deep learning tweaks. 
		\item We rigorously analyze and experimentally demonstrate that Soft-OMP and Soft-IHT effectively approximate their original counterparts under suitable conditions on the softsort temperature parameter.
		\item We unroll iterations of these algorithms into neural networks and propose the \emph{OMP-Net} and \emph{IHT-Net} architectures. We show that these networks are trainable with meaningful parameters, i.e., weights that capture latent structure within the data.
		\item We demonstrate that with weights as trainable parameters, OMP- and IHT-Net enable near noise-level recovery in heavily undersampled regimes, provided that the data exhibits sufficient structure. In particular, they are able to outperform, respectively, OMP and IHT in this scenario.
	\end{enumerate}
	
	\subsection{A glimpse at the literature}
	The advantages of algorithm unrolling have made it a popular model-based deep learning approach. 
	Several review papers (see, e.g., \cite{liang2020deep, monga2021algorithm, scarlett2022theoretical}) document this success story. Moreover, we highlight two particularly influential works, \cite{gregor2010learning} and \cite{colbrook2022difficulty}, that hold notable relevance to our study.
	
	The work \cite{colbrook2022difficulty} unrolls iterations of a primal-dual algorithm \cite{chambolle2011first} with recovery guarantees, thereby ensuring stability in neural networks while addressing fundamental limits of AI. On the other hand, \cite{gregor2010learning}, is one the earliest (if not the first) papers to explore algorithm unrolling, introducing the unrolling of ISTA (Iterative Shrinkage-Thresholding Algorithm) \cite{daubechies2004iterative}. ISTA follows a recursive relation similar to IHT but employs soft-thresholding, defined as
	\[S_\lambda(x)_j = \begin{cases}
		x_j - \lambda, & x_j \geq \lambda \\
		0, & -\lambda < x_j < \lambda \\
		x_j + \lambda, & x_j \leq -\lambda
	\end{cases},\]
	which is a continuous operator and thus compatible with neural networks (see \Cref{remark:differentiability}). The introduction of Learned ISTA (LISTA) inspired extensive research in the field, including theoretical analyses \cite{behboodi2022compressive,zhang2018ista}. However, OMP and IHT have received comparatively less attention, due to the inherent non-differentiability of their argsort-based operators. Unrolling IHT appeared in two nearly simultaneous works \cite{wang2016learning, xin2016maximal}, considering IHT with both $k$-sparse and the $\ell^0$-regularized hard-thresholding operators (terminology due to \cite{blumensath2008iterative}). Unlike soft-thresholding, the $\ell^0$-regularized component-wise operator defined as
	\[H_\lambda(x)_j = \begin{cases}
		x_j, & x_j \geq \lambda\ \text{or}\ x_j \leq -\lambda\\
		0, & -\lambda < x_j < \lambda
	\end{cases},\]
	is discontinuous (and thus non-differentiable), but its implementation in neural networks is straightforward and was done in \cite{xin2016maximal,wang2016learning} by using a continuous relaxation of the operator. However, these works do not directly address the non-differentiability of the $k$-sparse IHT, which includes argsort and is central to our work. The same holds for OMP in DeepPursuit, a deep reinforcement learning implementation of OMP \cite{chen2021deeppursuit}, and in the learned greedy method \cite{khatib2021learned}. The present paper builds upon the authors’ previous work \cite{mohammad2024omp}, introducing Soft-IHT and IHT-Net as a complementary contribution, while extending the theoretical and numerical scope of our earlier proposal on Soft-OMP and OMP-Net. 
	
	\subsection{Notation}
	\label{sec:notation}
	Here, we concisely introduce the necessary notations used throughout the paper, while defining additional ones in context as needed. We denote by $\mathbb{N}=\{1,2,3,\ldots\}$ the set of natural numbers and let $\mathbb{N}_0 = \{0\} \cup \mathbb{N}$. We use $[N]$ to represent the enumeration set $\{1, \dots, N\}$, let $[0:N]= \{0\} \cup [N]$, and use $[\ ]$ to denote an empty matrix. The notation $\|\cdot\|$ (without any subscript) denotes the $\ell^2$-norm of a vector. For a matrix $A \in \bC^{m \times N}$, we define its range as $\mathcal{R}(A) = \{y \in \bC^m: y = Ax, \ x \in \bC^N\}$, its transpose as $A^\top \in \bC^{N \times m}$, its conjugate transpose as $A^* \in \bC^{N \times m}$, and its $\ell^2$-operator norm as $\|A\| = \|A\|_{2 \to 2} = \sup_{\|x\| = 1}\|Ax\|$. We also adopt MATLAB-style notation, where $A(i, :)$ and $A(:, j)$ refer to the $i$th row and $j$th column of $A$, respectively. To distinguish between row and column representations, we write $A = [a_1, a_2, \dots]$ for a row-wise and $A = [a_1; a_2; \dots]$ for a column-wise arrangement. 
	
	The $s$th \emph{Restricted Isometry Constant} (RIC) of $A$, denoted by $\delta_s(A)$, is the smallest $\delta > 0$ satisfying $(1 - \delta)\|x\|^2 \leq \|Ax\|^{2} \leq (1 + \delta)\|x\|^{2} $, for all $s$-sparse vectors $x \in \bC^N$, i.e., those with $\|x\|_0 = \card(\supp(x)) \leq s$, where $\supp(x)$ denotes the support of the vector $x$ and $\card(\cdot)$ returns the cardinality of a set. For a set of indices $S$ with $\card(S) = s$, we denote by $x_S \in \bC^s$ the restriction of $x$ to the indices in $S$, and by $A_S \in \bC^{m \times s}$ the submatrix of $A \in \bC^{m \times N}$ consisting of the columns indexed by $S$. When clear from context, we may use the same notation to refer to zero-padded version of $x$, where entries outside $S$ (i.e., $[N]\backslash S$) are set to zero, preserving the original dimension. Additionally, the \emph{coherence} parameter of $A$ is given by $ \mu(A) = \max_{i, j \in [N],\ i \neq j}|\langle a_i, a_j\rangle| $ where $a_j = A(:, j)$ for $j \in [N]$ (see, e.g., \cite{foucart2013mathematical}).
	
	Throughout the paper, we use the tilde symbol `$ \sim $' to denote variables in the ``Soft" counterparts of OMP and IHT. For example, we represent the equivalent of $x$ in OMP or IHT by $\tilde{x}$ in Soft-OMP or Soft-IHT.
	
	\subsection{Outline of the paper}
	The rest of the paper is organized as follows. The next section provides preliminaries on greedy sparse recovery algorithms, namely OMP and IHT, and formalizes the argsorting issue in their unrolling. Then, \Cref{sec:unrolled_algorithms} introduces softsort and explains how this operator addresses the argsort challenge in OMP and IHT. In this section we also present Soft-OMP and Soft-IHT, along with theorems validating their approximation capabilities. To improve readability, proofs are deferred to \Cref{appendix:continuity,appendix:omp,appendix:iht}. \Cref{sec:greedy_networks} sets forth the neural network realization of Soft-$\ast$ algorithms. Numerical experiments validating our theoretical results are presented in \Cref{sec:unrolled_algorithms}, while their implications for the trainability of OMP- and IHT-Net are examined in \Cref{sec:numerics}. Finally, we conclude the paper with remarks on future research directions in \Cref{sec:conclusion}.
	
	\section{Preliminaries}
	\label{sec:preliminaries}
	In this section, we first contextualize the sparse recovery problem within our framework and provide an overview of the two greedy solvers that form the focus of this manuscript. Next, we elaborate on the argsort operator, which plays a critical role in hindering the neural network realization of these algorithms. This discussion sets the stage for \Cref{sec:unrolled_algorithms}, where we present one of our key contributions: differentiable alternatives to these algorithms, compatible with neural network learning.
	\subsection{Sparse recovery}
	The goal of sparse recovery is to reconstruct an approximately sparse signal $ x \in \bC^N $, observed through a linear transformation represented by a matrix $ A \in \bC^{m \times N} $, and potentially corrupted by some additive source of error $ e \in \bC^m $. This involves recovering $ x $ given observations $ y \in \bC^m $ assuming a linear measurement model
	\begin{equation}\label{eq:SR_model}
		y = Ax + e.
	\end{equation}
	Depending on the context, $ A $ is referred to as the design, mixing, sensing or dictionary matrix, and $ e $ models numerical or noise error. In compressive sensing or sparse representation onto overcomplete dictionaries, the system in \Cref{eq:SR_model} is underdetermined, i.e., $m < N$ or $m \ll N$. One can formulate a first recovery method in terms of the following \textit{non-convex} optimization problem
	\begin{equation}\label{eq:SR_problem}
		\hat{x} \in \argmin{z \in \bC^N}{\left\{\|y - Az\|^2\quad \text{subject\ to}\quad \|z\|_0 \leq k \right\}},
	\end{equation}
	where $ k $ represents the desired sparsity level. Since the sparse recovery problem above is computationally intractable (see, e.g., \cite{foucart2013mathematical}), one might pursue one of the following paths:
	\begin{enumerate}
		\item Relax the non-convex $ \ell^0 $-(pseudo) norm in favor of convex $ \ell^1 $-norm. This approach falls within the realm of convex optimization, leading to basis pursuit in its constrained form or LASSO-type programs in its unconstrained form.
		\item Impose the desired sparsity implicitly or locally through iterations of an iterative algorithm. This approach gives rise to iterative greedy and thresholding algorithms.
	\end{enumerate}
	In this paper, we focus on the latter as notable alternatives for convex optimization programs, with particular emphasis on OMP \cite{davis1994adaptive,pati1993orthogonal} and ($ k $-sparse) IHT \cite{blumensath2008iterative}. For reasons that will be discussed later, we classify IHT as a greedy algorithm, and consequently, study both OMP and IHT under the umbrella of greedy sparse recovery algorithms.
	
	Before discussing these algorithms, we briefly note that we are also concerned with the setting where prior information about the signal is available. Such information can be incorporated into a vector of ``weights" $ w \in \bR^N $, with $w \geq 0$, allowing emphasis to be placed on specific coefficients of the signal, thereby promoting certain structures within the signal (see, e.g., \cite{adcock2022sparse,bah2016sample,friedlander2011recovering,rauhut2016interpolation} and references therein). With this in mind, in the following sections we will also discuss briefly how $ w $ can be effectively incorporated into OMP and IHT, extending them to their weighted versions, WOMP and WIHT respectively. These extensions enable the recovery process to use prior knowledge about the signal structure. Similar weighted extensions have been studied in \cite{jo2013iterative, khajehnejad2011analyzing, mohammad2025greedy}.
	
	\subsection{Greedy algorithms}\label{subsec:greedy_algorithms}
	The core idea of a greedy strategy in solving computationally intractable problems (typically NP-hard) is to repeatedly select the local optimum when the global solution is unattainable within a feasible timeframe, which simply translates to``visit the nearest city!" in the well-known traveling salesman problem \cite{arora2009computational}. Similarly, since problem
	\eqref{eq:SR_problem} is NP-hard (see, e.g., \cite{foucart2013mathematical}), a greedy sparse recovery solver suggests to iteratively sparsify the solution at each stage. OMP incrementally builds a solution with the required sparsity level $k$, whereas IHT consistently enforces the same sparsity level throughout its iterations.
	\subsubsection{OMP} OMP is one of the most well-studied greedy algorithms for sparse recovery. It tackles  problem \eqref{eq:SR_problem} in a greedy fashion by decomposing the global optimization problem into a sequence of local least-squares optimization problems.
	
	At each iteration, in its simplest form, OMP identifies one additional index to add to the signal's support based on a ``greedy selection criterion". The algorithm then reconstructs the signal by iteratively updating the support and computing the corresponding signal coefficients. In this way, OMP ensures a one-to-one correspondence between the number of iterations and sparsity level of the reconstructed signal, effectively imposing the desired sparsity over the course of iterations. This property makes OMP particularly efficient in lower sparsity regimes, i.e., when $ s $ is small. The steps of OMP are outlined in \Cref{alg:OMP}.
	
	\begin{algorithm}[htb]
		\textbf{Input:} desired sparsity $k$, mixing matrix $ A \in \mathbb{C}^{m \times N} $ with $ \ell^2 $-normalized columns, observation vector $ y \in \mathbb{C}^m $. \\
		\vspace{-4mm}
		\begin{tcolorbox}[boxsep=3pt,
			left=1pt,
			right=1pt, colback=white, colframe=black, boxrule=0.5mm, opacityframe=0.3]
			Let $ x^{(0)} = 0 $ and $ S^{(0)} = \emptyset $.\\
			For $ n = 0, \dots, k - 1 $ repeat: 
			\begin{align}
				S^{(n + 1)} &= S^{(n)} \cup \{j^{(n + 1)}\},\quad j^{(n + 1)} \in \argmax{j \in [N]}{\left\{\left|\left(A^*(y - Ax^{(n)})\right)_j\right|\right\}}, \tag{OMP.1} \label{eq:OMP.1} \\
				x^{(n + 1)} &\in \argmin{z \in \mathbb{C}^N}{\left\{\|y - Az\|^2\quad \text{s.t.}\quad \text{supp}(z) \subseteq S^{(n + 1)}\right\}}. \tag{OMP.2}
			\end{align}
		\end{tcolorbox}
		\textbf{Output:} $ k $-sparse vector $ \hat{x} = x^{(k)} \in \mathbb{C}^N $.
		\caption{Orthogonal Matching Pursuit (OMP).}
		\label{alg:OMP}
	\end{algorithm}
	
	Here, $ S^ {(n + 1)} $ is the support of the signal at iteration $ n + 1 $, with the initializations $ x^{(0)} = 0 $ and $ S^{(0)} = \emptyset $. Throughout the paper, we refer to $ |A^*(y - Ax^{(n)})| $ as the ``greedy selection quantity". At each iteration of OMP, the greedy selection step introduces a competition among the column indices of $ A $ to join the support. The column with the strongest correlation to the residual $ r^{(n)}:= y - Ax^ {(n)} $ wins the competition and its index is added to the support.
	
	If any sort of prior knowledge about the signal is available, OMP can be extended to the weighted setting by appropriately modifying the greedy selection quantity, and introducing a degree of ``unfairness" in the competition. A semantic way of achieving this is to embed a weight vector (typically with values in the range $ [0, 1] $) as the diagonal entries of a diagonal matrix
	\begin{equation}\label{eq:weights}
		W = \text{diag}(w) \in \bR^{N \times N},\ \text{where}\ w \in \bR^N,\ w \geq 0.
	\end{equation}
	The modified greedy selection quantity then becomes $ |WA^*(y - Ax^{(n)})| $, effectively prioritizing certain indices over others. This idea has been explored in prior works \cite{bouchot2017multi,li2013weighted}.
	
	\subsubsection{IHT}\label{subsec:IHT} $ k $-sparse IHT was introduced in \cite{blumensath2008iterative} from an optimization transfer perspective, but can also be viewed as a fixed-point algorithm. For our purpose, however, it is more convenient to interpret IHT as combining two steps: (1) a descent step on the least-square loss function $ f(x) = \frac{1}{2}\|y - Ax\|^2 $, performed via  a Bregman-type iteration, and (2) a best $ k $-term approximation step to enforce sparsity.
	
	In the descent step, IHT moves the solution towards the gradient descent direction with an appropriate amount determined by a step size $ \eta > 0 $, resulting in
	\[u^{(n + 1)} = x^{(n)} - \eta \nabla_xf(x)\big|_{x = x^{(n)}} = x^{(n)} + \eta A^*(y - Ax^{(n)}).\]
	However, $ u^{(n + 1)} $ is generally non-sparse and thus, to impose the desired sparsity, IHT performs a greedy step via the best $ k $-term approximation 
	\[x^{(n + 1)} \in \argmin{z \in \bC^N,\ \|z\|_0 \le k}{\|z - u^{(n + 1)}\|^2},\]
	which explains why we place $ k $-sparse IHT as a greedy algorithm alongside OMP. This step is equivalent to applying the hard-thresholding operator $ H_k: \bC^N \to \bC^N $, that retains the top $ k $ entries of $ u^{(n + 1)} $ (in absolute value) and sets the rest to zero. More formally, let $ \rho: [N] \to [N] $ be a bijection that sorts elements of $ |u^{(n + 1)}| $ in non-increasing order $ |u_{\rho(1)}^{(n + 1)}| \geq |u_{\rho(2)}^{(n + 1)}| \geq \dots \geq |u_{\rho(N)}^{(n + 1)}| $. The support of $ x^{(n + 1)} $ is then given by $ S^{(n + 1)} = \{\rho(1), \dots, \rho(k)\} $, and the update becomes
	\[x^{(n + 1)} = H_k(u^{(n + 1)}) = u_{S^{(n + 1)}}^{(n + 1)}.\]
	Combining these two steps, $ k $-sparse IHT is summarized in \Cref{alg:IHT}.
	
	\begin{algorithm}[htbp]
		\textbf{Input:} desired sparsity $ k $, mixing matrix $ A \in \bC^{m \times N} $ with $ \ell^2 $-normalized columns, observation vector $ y \in \bC^m $, initial signal $ x^{(0)} \in \bC^N $, step size $ \eta > 0 $, number of iterations $\bar{n}$.\\
		\vspace{-4mm}
		\begin{tcolorbox}[boxsep=3pt,
			left=1pt,
			right=1pt, colback=white, colframe=black, boxrule=0.5mm, opacityframe=0.3]
			For $ n = 0, \dots, \bar{n} - 1 $ repeat:
			\begin{equation}
				x^{(n + 1)} = H_k(x^{(n)} + \eta A^*(y - Ax^{(n)})). \tag{IHT}
			\end{equation}
		\end{tcolorbox}
		\textbf{Output:} $ k $-sparse vector $ \hat{x} = x^{(\bar{n})} \in \bC^N $.
		\caption{Iterative Hard Thresholding (IHT).}
		\label{alg:IHT}
	\end{algorithm}
	
	To emphasize or suppress specific entries based on provided prior knowledge on the signal, one can replace the hard-thresholding operator $ H_k(\cdot) $ with the weighted hard-thresholding operator $ H_{w, k}(\cdot) $. This operator relies on the bijection $ \rho: [N] \to [N] $ associated this time with the non-increasing rearrangement of $ |Wu^{(n + 1)}| $, where $ W $ is the diagonal matrix defined in \Cref{eq:weights}, with the ordering $ |w_{\rho(1)}u_{\rho(1)}^{(n + 1)}| \geq |w_{\rho(2)}u_{\rho(2)}^{(n + 1)}| \geq \dots \geq |w_{\rho(N)}u_{\rho(N)}^{(n + 1)}| $. As in the standard case, the support of $ x^{(n + 1)} $ is then defined as $ S^{(n + 1)} = \{\rho(1), \dots, \rho(k)\} $, and the update becomes $ x^{(n + 1)} = H_{w, k}(u^{(n + 1)}) = u_{S^{(n + 1)}}^{(n + 1)} $. The idea of weighted IHT has been previously proposed in \cite{jo2013iterative}, albeit with a technical variation regarding weighted cardinality, which we omit purposefully to avoid unnecessary complications.
	
	\subsection{The (arg)sorting issue in unrolling greedy algorithms}
	Algorithm unrolling is the process of mapping iterations of an algorithm onto the layers of a neural network, with each layer playing the role of one iteration of the original algorithm. Treating certain parameters of the neural network as trainable, the objective of algorithm unrolling is to optimize the performance of the algorithm with respect those parameters through a gradient-based learning process for the neural network, possibly with fewer layers (iterations) than the original algorithm, thereby reducing the computational cost. This paradigm offers significant benefits on both the neural network and algorithmic fronts. Since algorithms typically stem from well-defined models, they bring interpretability to the neural network. Moreover, if the original algorithm enjoys recovery and convergence guarantees, these properties can often be transferred to the neural network. This alleviates common concerns about neural networks, such as their black-box nature, potential instability and lack of interpretability \cite{colbrook2022difficulty,kutyniok2024mathematics}.
	
	Despite the appeal of algorithm unrolling, not all algorithms can be readily unrolled. Key elements acting against ``unrollability" of some algorithms include the presence of non-differentiable operators (see \Cref{remark:differentiability}) and implicit dependencies between the algorithm's steps, which hinders the flow of gradients during backpropagation in gradient-based learning. Notably, greedy algorithms like OMP and IHT are illustrative examples of these challenges as they rely on the operator argsort. 
	
	More precisely, for an input vector $ v \in \bR^N $ the operators $ \text{sort}: \bR^N \to \bR^N $ and $ \text{argsort}: \bR^N \to \mathcal{S} \subset [N]^N $ (assuming non-increasing ordering) are defined as follows:
	\begin{equation}\label{eq:sorting}
		\begin{aligned}
			&\sort(v)=(v_{\varrho_1}, v_{\varrho_2}, \dots, v_{\varrho_N}),\quad v_{\varrho_1} \geq v_{\varrho_2} \geq \dots \geq v_{\varrho_N}, \\
			&\argsort(v)=(\varrho_1, \varrho_2, \dots, \varrho_N),
		\end{aligned}
	\end{equation}
	where $ \mathcal{S} $ is the set of all permutations of $[N]$. \Cref{fig:softsort} graphically illustrates these operators in a toy example with $N = 2$ and $v_2 = 1$ fixed along with their approximate relaxations using softsort operator, to be introduced in \Cref{sec:unrolled_algorithms}. Throughout the paper, we assume there are no ties between the elements of the vector $v$, i.e., $v_i \neq v_j,\ i \neq j$ for all $i, j = 1, \dots, N $. This assumption is justified by the observation that, when running OMP and IHT, ties are very rare in practice due to the presence of numerical round-off errors.
	
	The argmax operator in OMP and the hard-thresholding operator in IHT both depend on argsort, which is a discontinuous operator with abrupt kinks and jumps, rendering these algorithms entirely non-differentiable. These algorithms involve the crucial step $ \mathcal{J}^{(n + 1)} =  \argsort(v^{(n + 1)})$ within their iterations, where $ v^{(n + 1)} $ is the vector to be sorted, i.e., $ v^{(n + 1)} = |A^*(y - x^{(n)})| $ and $ v^{(n)} = |u^{(n + 1)}| $ for OMP and IHT, respectively. IHT selects the first $ k $ indices of $ \mathcal{J}^{(n + 1)} $ and passes the corresponding entries to $ x^{(n + 1)} $. This breaks the gradient path of $ x^{(n + 1)} $ with respect to any parameter (e.g., $ \eta $), due to the argsort operator. This issue is even more pronounced for OMP. After computing $ \mathcal{J}^{(n + 1)} $, it adds the first index to the current support $ S^{(n)} $ to form the updated support $ S^{(n + 1)} $. The next step solves a least-squares problem restricted to the indices in $ S^{(n + 1)} $. This introduces not only non-differentiability due to application of the argsort operator, but also an implicit connection between steps of the algorithm, which is doubly problematic.
	\begin{figure}[t!]
		\centering
		\includegraphics[width = 0.75\linewidth]{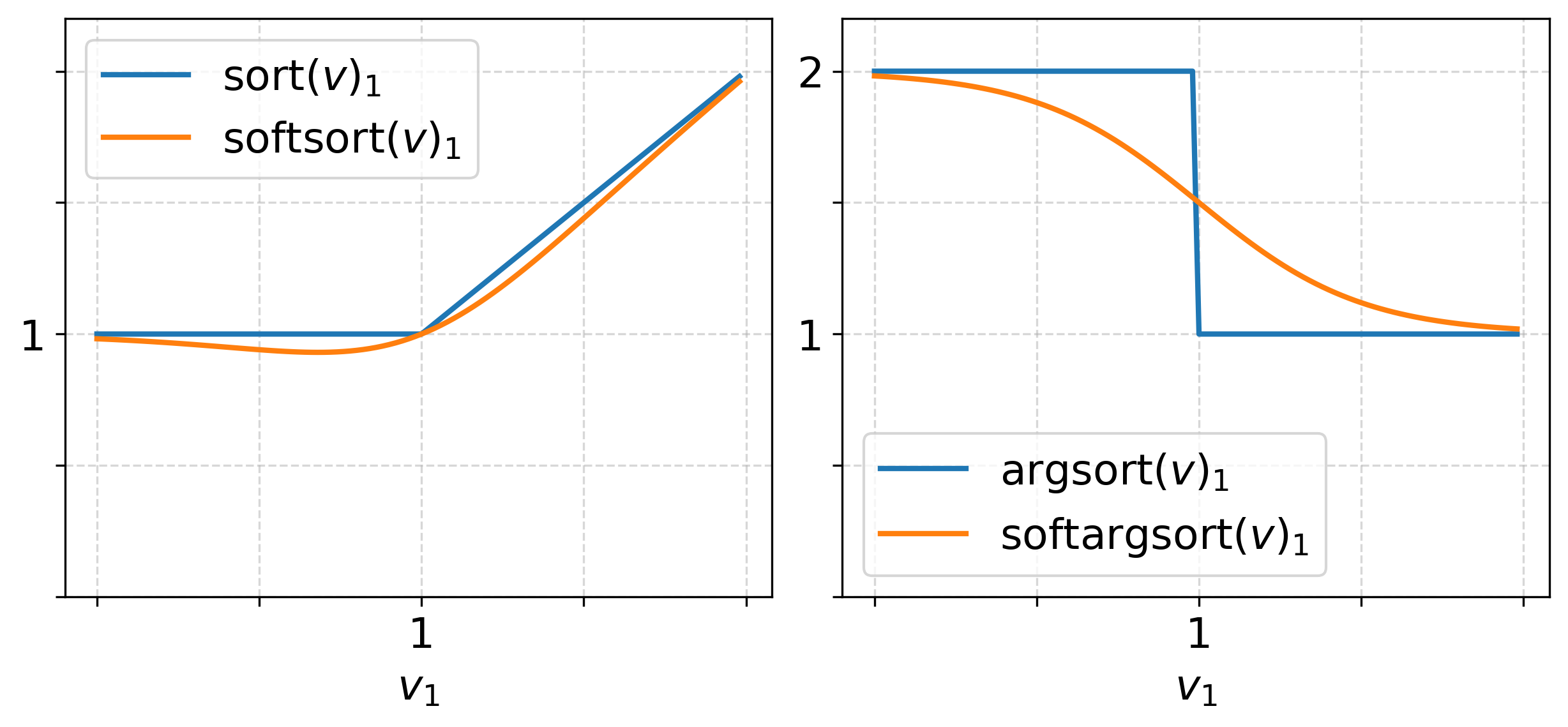}
		\caption{First element of the operators sort, softsort, argsort and softargsort applied to the two-dimensional vector $v = (v_1, v_2)$, as defined in \Cref{eq:sorting,eq:permutation,eq:softsort}, shown as a function of $v_1$ for fixed $v_2 = 1$ (we choose $\tau = 0.25$ in \Cref{eq:softsort}).}
		\label{fig:softsort}
	\end{figure}
	\begin{remark}[Algorithmic differentiability and subgradients]\label{remark:differentiability}
		In neural network optimization, certain types of non-differentiability such as corner singularities, are often manageable by employing subgradient generalizations, i.e., substituting the gradient at the non-differentiable point with an appropriate subgradient. Examples are functions such as $f(x) = |x|$ and $f(x) = \text{ReLU}(x)$ at $x = 0$ or $\sort(x)$ as shown in \Cref{fig:softsort}. However, the more problematic form of non-differentiability in neural networks is discontinuity, as seen in the argsorting operator (see \Cref{fig:softsort}), which cannot be treated using such standard remedies.
	\end{remark}
	
	\section{Unrolled greedy algorithms}
	\label{sec:unrolled_algorithms}
	As seen previously, the primary obstacle hindering unrolling greedy algorithms is the presence of argsort operator within their iterations. However, to address this challenge we must first make the connections between algorithmic steps explicit to preserve gradient flow across iterations. The ``missing link" to bridging non-differentiable algorithms with implicit steps and differentiable algorithms with explicit gradient path lies in the notion of permutation, intrinsic to (arg)sorting.
	
	More formally, (arg)sorting can be characterized by a permutation matrix $ P_{\argsort(\cdot)} \in \bR^{N \times N} $, as for a vector $ v \in \bR^N $
	\begin{equation}\label{eq:permutation}
		\begin{aligned}
			&\sort(v) = P_{\argsort(v)}v,\quad \argsort(v) = P_{\argsort(v)}\bar{1}_N,\\
			&P_{\argsort(v)}(i,j) = \begin{cases}
				1, & j = \argsort(v)_i, \\
				0, & \text{otherwise}
			\end{cases},
		\end{aligned}
	\end{equation}
	where $ \bar{1}_N = (1, \ldots, N)^\top$ is the vector corresponding to the set $[N]$. Note that the matrix operator $ P_{\argsort(\cdot)} $ is well-defined for a fixed permutation in $ \bR^N $; otherwise, argsort is fundamentally a nonlinear operator.
	
	Reinterpreting steps of OMP and IHT through a projection-based lens allows us to explicitly express algorithmic steps in terms of permutation matrices, but this alone does not resolve the non-differentiability issue. The key remaining task is to replace these permutation matrices with differentiable proxies that approximate them with a controllable degree of accuracy. In what follows, we introduce a reasonable choice for this proxy---\emph{softsort}---and use it to construct differentiable approximations of OMP and IHT, which we theoretically justify as valid alternatives to the exact algorithms.
	
	\subsection{Softsorting}
	Softsort, introduced in \cite{prillo2020softsort}, provides a continuous relaxation of the argsort operator (see \Cref{fig:softsort}) by approximating $ P_{\argsort(\cdot)} $ with a permutation matrix $ \tilde{P}_{\argsort(\cdot)} $. For a column vector $ v \in \bR^N $, softsort is defined as
	\begin{equation}\label{eq:softsort}
		\tilde{P}_{\argsort(v)} = \softsort(v) = \softmax\left(\frac{-|\sort(v)\mathds{1}^\top - \mathds{1}v^\top|}{\tau}\right),
	\end{equation}
	where the softmax function, applied row-wise is given by
	\[\softmax(a)_j = \frac{e^{a_j}}{\sum_{i = 1}^{N}e^{a_i}},\quad j \in [N],\quad a \in \bR^N,\]
	$ \mathds{1} = (1, \dots, 1)^\top \in \bR^N $ is an all-one vector, and $ \tau $ is a \emph{temperature parameter} that controls the approximation accuracy. Softsort applies softmax on the negative distance matrix derived from $ v $ and $ \sort(v) $, emphasizing locations where $ v $ and $ \sort(v) $ coincide, corresponding to the positions of elements in $ \argsort(v) $ (see \hyperref[props:argmax]{\Cref{prop:properties}\ref{props:argmax}}). With sufficiently small $ \tau $, softsort produces dominant values at these positions and small values at others, effectively approximating the discrete argsort operator. This is exemplified below.
	\begin{example}
		Let $ x = (1, 3, 0.5, 4.5)^\top $.
			\begin{align*}
				P_{\argsort(x)} & = \begin{bmatrix}
					0 & 0 & 0 & 1\\
					0 & 1 & 0 & 0\\
					1 & 0 & 0 & 0\\
					0 & 0 & 1 & 0
				\end{bmatrix},\\
				\sort(x) & = P_{\argsort(x)}x = (4.5, 3, 1, 0.5)^\top,\\
				\argsort(x) & = P_{\argsort(x)}\bar{1}_N = (4, 2, 1, 3)^\top,
			\end{align*}
			which for $\tau = 0.75$ yields the approximation
			\begin{align*}
				\tilde{P}_{\argsort(x)} & = \mathrm{softsort}(x, \tau = 0.75) = \begin{bmatrix}
					0.0082 & 0.1177 & 0.0042 & \mathbf{0.8699}\\
					0.0560 & \mathbf{0.8061} & 0.0287 & 0.1091\\
					\mathbf{0.6280} & 0.0436 & 0.3224 & 0.0060\\
					0.3304 & 0.0230 & \mathbf{0.6435} & 0.00311
				\end{bmatrix}, \\
				\tilde{P}_{\argsort(x)}x & = (4.2780, 2.9797, 0.9467, 0.7350)^\top, \\
				\tilde{P}_{\argsort(x)}\bar{1}_N & = (3.7358, 2.1909, 1.7062, 2.3193)^\top.
			\end{align*}
	\end{example}
	Softsort as in \Cref{eq:softsort} is an instance of a ``unimodal row-stochastic matrix". It possesses several desirable properties listed below.
	\begin{proposition}[Properties of the softsort operator \cite{prillo2020softsort}]\label{prop:properties}
		Softsort satisfies:
		\begin{enumerate}[leftmargin=1.2cm, label= (\alph*)]
			\item\label[property]{props:nonneg} Non-negativity: $ \tilde{P}_{\argsort(v)}(i, j) \geq 0 $.
			\item\label[property]{props:row_affine} Row affinity: Because of softmax, elements in each row of $ \softsort(v) $ sum to 1, i.e., $ \sum_{j = 1}^{N}\tilde{P}_{\argsort(v)}(i, j) = 1,\ \forall i \in [N]. $
			\item\label[property]{props:asym} Asymptotic behavior: $ \softsort(v) \to P_{\argsort(v)} $, as $ \tau \to 0 $.
			\item\label[property]{props:argmax} Argmax property: $ \argsort(v) = \argmax{}{(\tilde{P}_{\argsort(v)})} $, where argmax is applied row-wise.
			\item\label[property]{props:perm_eq} Permutation equivariance: $ \softsort(v) = \softsort(\sort(v))P_{\argsort(v)} $, which can be viewed as a direct consequence of \Cref{props:argmax}.
		\end{enumerate}
	\end{proposition}
	In addition to the above properties, softsort is also a continuous operator. While continuity of softsort was assumed without formal discussion in \cite{prillo2020softsort}, we establish a stronger result in \Cref{prop:continuity} below. Its proof can be found in \Cref{appendix:continuity}.
	\begin{proposition}[Lipschitzness of softsort]\label{prop:continuity}
		$ \softsort(\cdot): \bR^N \to \bR^{N \times N} $ is a continuous function. Moreover, the function defined by each row of the softsort operator, i.e., $\bR^N \ni v \mapsto \softsort(v)(i, :) $, is $ M $-Lipschitz with respect to the $\ell^2$-norm, where $M = (\sqrt{N} + 1)/\tau $.
	\end{proposition}
	Continuity of softsort, and hence its algorithmic differentiability (see \Cref{remark:differentiability}), enables neural network implementations of the argsort operator. Building on this foundation, in the following subsections we use the approximate permutation matrix obtained from softSort to unfold OMP and IHT, respectively. However, the Lipschitz constant $ M $ in the proposition above tends to infinity as $ \tau \to 0 $, indicating that continuity of softsort deteriorates for small values of $ \tau $. This introduces a trade-off for $ \tau $ between the continuity of softsort and its approximation, as highlighted in \hyperref[props:asym]{\Cref{prop:properties}\ref{props:asym}}, which states that softsort more closely approximates the true sorting operator as $ \tau $ decreases. Thus, selecting an appropriate $ \tau $ is crucial in balancing theory and practice in the upcoming sections. Theoretical constraints impose an upper bound for $ \tau $, beyond which the desired approximation is not achieved, while practical considerations in training the neural networks demand for higher values of $ \tau $ (see \Cref{sec:numerics}).
	
	\subsection{Soft-OMP}
	The first step to unroll OMP (see \Cref{alg:OMP}) is to reinterpret its iterations through a projection (permutation-based) lens. This not only enables the direct application of soft sorting but also explicitly reveals the connection between successive OMP iterations, facilitating efficient gradient flow.
	
	At iteration $ n + 1 $, OMP computes the projection of the observation signal $ y $ onto an $ (n + 1) $-dimensional subspace of $ \mathcal{R}(A) $, spanned by the columns iteratively selected during the greedy selection step, i.e., $ A_{S^{(n + 1)}} \in \bC^{m \times (n + 1)} $. This subspace expands over iterations, progressively improving the approximation. This process is equivalent to applying a permutation matrix $ \Pi^{(n + 1)} \in \bR^{(n + 1) \times N} $ to the column space of $ A $, yielding $ A(\Pi^{(n + 1)})^\top =: B^{(n + 1)} $, where rows of $ \Pi^{(n + 1)} $ correspond to $ n + 1 $ selected indices in $ S^{(n + 1)} $, i.e., in each row there is a 1 at the location of the selected index and 0 elsewhere. Since OMP selects one additional column at each iteration, the permutation matrix $ \Pi^{(n + 1)} $ is constructed incrementally by appending a new row to the previous iteration’s permutation matrix $ \Pi^{(n + 1)} = [\Pi^{(n)}; P^{(n + 1)}(1, :)] $, where $ P^{(n + 1)} := P_{\argsort(v^{(n + 1)})} $ (as in \Cref{eq:permutation}), with $ v^{(n + 1)} = |A^*(y - Ax^{(n)})| $. Once this subspace is established, the next step is to solve the least-squares problem within it, followed by mapping the resulting $ (n + 1) $-dimensional signal back to the original $ N $-dimensional space. These steps define the following algorithm, which we term \emph{pOMP} (projection-based OMP), serving as an intermediary between OMP and its unrolled counterpart.
	
	\begin{algorithm}[htbp]
		\textbf{Input:} desired sparsity $k$, mixing matrix $ A \in \bC^{m \times N} $ with $ \ell^2 $-normalized columns, observation vector $ y \in \bC^m $.\\
		\vspace{-4mm}
		\begin{tcolorbox}[boxsep=3pt,
			left=1pt,
			right=1pt, colback=white, colframe=black, boxrule=0.5mm, opacityframe=0.3]
			Let $ x^{(0)} = 0 $ and $ \Pi^{(0)} = [\ ] $.\\
			For $ n = 0, \dots, k - 1 $ repeat:
			\begin{align*}
				&\begin{cases}
					v^{(n + 1)} = \left|A^*(y - Ax^{(n)})\right|, \\
					P^{(n + 1)} = P_{\argsort(v^{(n + 1)})}, \\
					\Pi^{(n + 1)} = \left[\Pi^{(n)}; P^{(n + 1)}(1, :)\right],
				\end{cases} \tag{pOMP.1}\\
				&\begin{cases}
					B^{(n + 1)} = A(\Pi^{(n + 1)})^\top, \\
					w^{(n + 1)} \in \argmin{z \in \mathbb{C}^{(n + 1)}}{\|y - B^{(n + 1)}z\|^2}, \\
					x^{(n + 1)} = (\Pi^{(n + 1)})^\top w^{(n + 1)}.
				\end{cases} \tag{pOMP.2}
			\end{align*}
		\end{tcolorbox}
		\textbf{Output:} $ k $-sparse vector $ \hat{x} = x^{(k)} \in \bC^N $.
		\caption{Projection-based OMP (pOMP).}
		\label{alg:pOMP}
	\end{algorithm}
	
	It is evident that pOMP generates the same sequence of output signals as OMP. In pOMP, the relationship between the vector to be sorted, $ v^{(n + 1)} $, and the output signal at that iteration, $ x^{(n + 1)} $, is established through the permutation matrix $ \Pi^{(n + 1)} $ (and $ P^{(n + 1)} $). The key obstacle to differentiability is the sorting-based selection operator $ P_{\argsort(\cdot)} $, which is inherently non-differentiable and blocks gradient flow. To address this, we approximate the permutation matrix $ P_{\argsort(\cdot)} $ in a differentiable manner by replacing it with a softsort proxy. This modification immediately yields a differentiable approximation of OMP, which we refer to as \emph{Soft-OMP}, because of the softsort operator.  
	
	\begin{algorithm}[htbp]
		\textbf{Input:} desired sparsity $k$, mixing matrix $ A \in \bC^{m \times N} $ with $ \ell^2 $-normalized columns, observation vector $ y \in \bC^m $,  temperature parameter $ \tau $ of softsort.\\
		\vspace{-4mm}
		\begin{tcolorbox}[boxsep=3pt,
			left=1pt,
			right=1pt, colback=white, colframe=black, boxrule=0.5mm, opacityframe=0.3]
			Let $ \tilde{x}^{(0)} =0 $ and $ \tilde{\Pi}^{(0)} = [\ ] $.\\
			For $ n = 0, \dots, k - 1 $ repeat:
			\begin{align*}
				&\begin{cases}
					\tilde{v}^{(n + 1)} = \left|A^*(y - A\tilde{x}^{(n)})\right|, \\
					\tilde{P}^{(n + 1)} = \softsort(\tilde{v}^{(n + 1)}),\quad \text{(see, \Cref{remark:softsort-simplified}),} \\
					\tilde{\Pi}^{(n + 1)} = \left[\tilde{\Pi}^{(n)}; \tilde{P}^{(n + 1)}(1, :)\right],
				\end{cases} \tag{Soft-OMP.1}\\
				&\begin{cases}
					\tilde{B}^{(n + 1)} = A(\tilde{\Pi}^{(n + 1)})^\top, \\
					\tilde{w}^{(n + 1)} \in \argmin{z \in \mathbb{C}^{(n + 1)}}{\|y - \tilde{B}^{(n + 1)}z\|^2}, \\
					\tilde{x}^{(n + 1)} = (\tilde{\Pi}^{(n + 1)})^\top \tilde{w}^{(n + 1)}.
				\end{cases} \tag{Soft-OMP.2}
			\end{align*}
		\end{tcolorbox}
		\textbf{Output:} Approximately $k$-sparse vector $ \hat{x} = \tilde{x}^{(k)} \in \bC^N $.
		\caption{Soft-OMP.}
		\label{alg:Soft-OMP}
	\end{algorithm}
	
	\begin{remark}
		Although we choose softsort for its closed-form and intuitive formulation, it is worth noting that the core idea in this paper is to provide a suitable differentiable approximation for $ P_{\argsort(\cdot)} $ and replacing it with the sofsort operator is just one of several possible techniques. See \Cref{sec:conclusion}, for other alternatives.
	\end{remark}
	\begin{remark}\label{remark:softsort-simplified}
		Since only the first row of the softsort matrix is of interest at each iteration of \Cref{alg:Soft-OMP}, one can show
			\[\begin{aligned}
				&\tilde{P}^{(n + 1)}(1, :) = \softsort(\tilde{v}^{(n + 1)})(1, :) = \softmax\left(\frac{-\left|\tilde{v}^{(n + 1)}_{\tilde{j}^{(n + 1)}}\mathds{1} - \tilde{v}^{(n + 1)}\right|}{\tau}\right)\\
				&= \softmax\left(\frac{-\tilde{v}^{(n + 1)}_{\tilde{j}^{(n + 1)}}\mathds{1} + \tilde{v}^{(n + 1)}}{\tau}\right) = \softmax\left(\tilde{v}^{(n + 1)}/\tau\right),
			\end{aligned}\]
			where $\tilde{j}^{(n + 1)} = \argmax{j \in [N]}{\tilde{v}_j^{(n + 1)}}$ and thus we used the fact that $\tilde{v}^{(n + 1)}_{\tilde{j}^{(n + 1)}}\mathds{1} \geq \tilde{v}^{(n + 1)}$, and that for all $v \in \bR^N$ and constant $c \in \bR$, we have $\softmax(v + c\mathds{1}) = \softmax(v)$. As a result, one can simply apply softmax on the vector $v^{(n+1)}$, which is a more memory-efficient and computationally optimized implementation. Despite, we used softsort in \Cref{alg:Soft-OMP} to be consistent with the rest of the paper. 
		
	\end{remark}
	The question that remains to be studied is how well Soft-OMP can approximate OMP. As noted earlier, softsort approximates $ P_{\argsort(\cdot)} $ with accuracy controlled by its temperature parameter $ \tau $. The following theorem formalizes this idea, relating Soft-OMP to pOMP (and consequently OMP) under an appropriate condition on $ \tau $. 
	\begin{theorem}[Soft-OMP is a good approximation to OMP]
		\label{thm:omp_main}
		For $n \in \bN_0$ fixed, let $ v^{(n)} \in \bR^N$, $ \Pi^{(n)} \in \bR^{n \times N} $ and $ x^{(n)} \in \bC^{N} $ be \graysout{sequences} generated by \Cref{alg:pOMP}; and likewise $ \tilde{v}^{(n)} \in \bR^N$, $ \tilde{\Pi}^{(n)} \in \bR^{n \times N} $ and $ \tilde{x}^{(n)} \in \bC^{N} $ be \graysout{sequences} generated by \Cref{alg:Soft-OMP}. Then, assuming $x^{(0)} = \tilde{x}^{(0)}$, the following hold:
		\begin{enumerate}[leftmargin=1.2cm, label=(\roman*)]
			\item Asymptotic convergence: As $ \tau \to 0 $, we have $ \tilde{\Pi}^{(n)} \to \Pi^{(n)}$ and $\tilde{x}^{(n)} \to x^{(n)} $.
			\item Non-asymptotic convergence: Choose
			\[0 < \epsilon < g^{(1:n)}/2\|A\|,\]
			where
			\begin{equation}\label{eq:omp-gap}
				\begin{aligned}
					g^{(1:n)} &:= \min_{i \in [n]}g^{(i)}, \\
					g^{(i)} &:= \min_{j \in [N], j\neq j^{(i)}}\left|v_j^{(i)} - v_{j^{(i)}}^{(i)}\right|\ \text{with}\ j^{(i)} := \argmax{}{(v^{(i)})},
				\end{aligned}
			\end{equation}
			are OMP's ``global min-gap" and ``local min-gap", respectively. Assume that $\tau$, the temperature parameter of $\text{softsort}$, satisfies the condition
				\begin{equation}\label{eq:tau-omp}
					\tau \leq \left(g^{(1:n)} - 2\|A\|\epsilon\right)/\log\left(C^{(n)}/\epsilon\right),
				\end{equation}
				where 			
				\begin{equation*}
					C^{(n)} = \sqrt{2n}(N - 1)\left(\frac{\sqrt{1 - \delta_n(A)} + 2\sqrt{n}\|A\|}{1 - \delta_n(A)}\right)\|y\|,
				\end{equation*}
				and $ \delta_n(A) < 1 $ is the $ n $th restricted isometry constant of $ A $. Then
				\[\argmax{}{\left(\Pi^{(n)}\right)} = \argmax{}{\left(\tilde{\Pi}^{(n)}\right)}, \]
				where argmax is applied row-wise, and
			\[\max_{i \in [0:n]}\left\|x^{(i)} - \tilde{x}^{(i)}\right\| \leq \epsilon.\]
		\end{enumerate}
	\end{theorem}
	\begin{proof}[Proof (sketch)]
		We present a proof sketch here. For the complete proof, see \Cref{appendix:omp}.
		\begin{itemize}[leftmargin=1.2cm]
			\item We first ensure that Soft-OMP selects the same indices as OMP. Given the same initialization, $ x^{(0)} = \tilde{x}^{(0)} $, we establish conditions under which Soft-OMP preserves the order of elements across iterations (in fact, only the maximum is sufficient). This guarantees that $ \argmax{}{(\Pi^{(n)})} = \argmax{}{(\tilde{\Pi}^{(n)})} $.
			\item Once this is established, we formalize approximation of $ \Pi^{(n)} $ by $ \tilde{\Pi}^{(n)} $, providing upper bounds for $ \|\Pi^{(n)} - \tilde{\Pi}^{(n)}\| $ and $\|\tilde{\Pi}^{(n)}\| $, in terms of $ \tau $ in \Cref{lemma:frob_bounds}.
			\item Finally, we relate $ \|x^{(n)} - \tilde{x}^{(n)}\| $ to the bound derived for $ \|\Pi^{(n)} - \tilde{\Pi}^{(n)}\| $, analyzing the deviations introduced in $ B^{(n)} = A_{S^{(n)}} = A(\Pi^{(n)})^\top $ through the sensitivity analysis of least-squares, establishing this in \Cref{th:Trefethen,lemma:distance_bound}. We then prove the claim by induction on $ n $.
		\end{itemize}
	\end{proof}
	In words, \Cref{thm:omp_main} states that to achieve a desired accuracy $ \epsilon $, the gap between the largest and the second-largest elements of $ v^{(i)} $ must be sufficiently large. This ascertains that the maximum index of $ v^{(i)} $ at each iteration of Soft-OMP matches that of OMP, meaning Soft-OMP retrieves the same indices as OMP, i.e., $ \argmax{}{(\Pi^{(n)})} = \argmax{}{(\tilde{\Pi}^{(n)})} $ (Soft-OMP follows OMP's indices). Furthermore, if $ \tau $ is small enough relative to the gap between elements of Soft-OMP, the column selector $ \tilde{\Pi}^{(n + 1)} $ will closely approximate $ \Pi^{(n + 1)} $, the column selector in OMP. As a result, the signal recovered at each iteration of Soft-OMP remains within an $ \epsilon $-distance of the signal obtained by OMP, ensuring that Soft-OMP approximates OMP to the desired $ \epsilon $-accuracy. We conclude by making a few remarks.
	\begin{remark}\label{remark:gaps}
		We note that $ g^{(i)},\ \forall i \in [n] $ , is not a tunable parameter but rather depends on the underlying ``physics" of the problem, i.e., class of signals considered, noise, etc. (see \Cref{eq:SR_model}), and in a practical scenario one might have limited control over it. In the first iteration, if OMP and Soft-OMP are initialized with the same signal, i.e., $ x^{(0)} = \tilde{x}^{(0)} $, it suffices for $ g^{(1)} $ to be strictly positive in order to have $\tilde{j}^{(1)}=j^{(1)}$ (i.e., to choose the correct index). This is, in turn, equivalent to assuming that there are no ties in step \eqref{eq:OMP.1} of the first iteration of OMP. However, as iterations progress, approximation errors accumulate, requiring $ g^{(i)} $ to be sufficiently large for Soft-OMP to continue closely following OMP. Hence, the ``no tie'' assumption is no longer sufficient from the second iteration onwards and the ``physics'' of the problem has a more direct impact on the convergence of Soft-OMP to OMP.
	\end{remark}
	\begin{remark}\label{remark:recovery_gurantee}
		In light of recovery guarantee results available in the literature for OMP (see, e.g., \cite{zhang2011sparse}) and the inequality
		$ \|\tilde{x}^{(n)} - x\| \leq \|\tilde{x}^{(n)} - x^{(n)}\| + \|x^{(n)} - x\|, $
		where $ x \in \bC^N $ is the true signal in \Cref{eq:SR_model}, \Cref{thm:omp_main} also implies a recovery guarantee for Soft-OMP.
	\end{remark}

	\subsection{Soft-IHT}
	To unroll IHT (see \Cref{alg:IHT}), we follow the same approach as for OMP and introduce a projection-based variant for IHT called \emph{pIHT}. We define $ Q^{(n + 1)} = P_{\argsort(v^{(n + 1)})} \in \bR^{N \times N} $, where $ v^{(n + 1)} := |u^{(n + 1)}| $ is the quantity to be sorted and $ u^{(n + 1)} = x^{(n)} + \eta A^*(y - Ax^{(n)}) $ defining each iteration of IHT. Then the first $ k $ rows of $ Q^{(n + 1)} $ determine the vectors $ \{e_{\rho(i)}\}_{i = 1}^k $, spanning the $ k $-dimensional signal subspace $ \Sigma_k^N \subset \bR^N $, where $ e_{\rho(i)} $ is a canonical (one-hot) vector with 1 at position $ \rho(i) $ and 0 elsewhere, and $ \rho:[N] \to [N] $ is the bijection defined by the hard thresholding operator $ H_k(\cdot) $ introduced in \Cref{subsec:IHT}. Consequently, the output of the hard thresholding operator simply expands onto these vectors using the top-$ k $ elements of $ u^{(n + 1)} $ in magnitude as expansion coefficients
	\[x^{(n + 1)} = H_k\left(u^{(n + 1)}\right) = \sum_{i = 1}^{k}u^{(n + 1)}_{\rho(i)}Q^{(n + 1)}(i, :).\]
	This sum can be expressed more compactly as $ x^{(n + 1)} = q^{(n + 1)} \odot u^{(n + 1)} $, where $ \odot $ denotes the Hadamard (or componentwise) product and
	\[q^{(n + 1)} = \sum_{i = 1}^kQ^{(n + 1)}(i, :),\]
	the filter ($0$/$1$ mask) that selects the top-$ k $ elements of $ u^{(n + 1)} $ in magnitude. In other words, $ H_k(\cdot) $ acts as a multiplication operator characterized by $ q^{(n + 1)} $. This leads to the following algorithm.
	
	\begin{algorithm}[htbp]
		\textbf{Input:} desired sparsity $ k $, mixing matrix $ A \in \bC^{m \times N} $ with $ \ell^2 $-normalized columns, observation vector $ y \in \bC^N $, initial signal $ x^{(0)} \in \bC^N $, step size $ \eta > 0 $, number of iterations $\bar{n}$.\\
		
		\vspace{-4mm}
		\begin{tcolorbox}[boxsep=3pt,
			left=1pt,
			right=1pt, colback=white, colframe=black, boxrule=0.5mm, opacityframe=0.3]
			Let $ Q^{(0)} = [\ ] $.\\
			For $ n = 0, \dots, \bar{n} - 1 $ repeat:
			\begin{equation*}
				\begin{cases}
					u^{(n + 1)} = x^{(n)} + \eta A^*(y - Ax^{(n)}) \\
					Q^{(n + 1)} = P_{\argsort(v^{(n + 1)})},\quad v^{(n + 1)} = \left|u^{(n + 1)}\right|\\
					q^{(n + 1)} = \sum_{i = 1}^kQ^{(n + 1)}(i, :) \\
					x^{(n + 1)} = q^{(n + 1)} \odot u^{(n + 1)}.
				\end{cases} \tag{pIHT}
			\end{equation*}
		\end{tcolorbox}
		\textbf{Output:} $ k $-sparse vector $ \hat{x} = x^{(\bar{n})} \in \bC^N $.
		\caption{Projection-based IHT (pIHT).}
		\label{alg:pIHT}
	\end{algorithm}
	
	Similar to OMP, achieving a neural network-compatible implementation of IHT requires replacing the permutation matrix $ Q^{(n + 1)} $ in \Cref{alg:pIHT} with an approximate permutation matrix computed via softsort. In doing so, the rows of the softsort matrix serve as approximate canonical vectors, with the first $ k $ vectors spanning an anisotropic $ N $-dimensional subspace of $ \bR^N $, stretched in the directions $ \rho(1), \dots, \rho(k) $ (assuming $ \tau $ is not too large), thanks to \hyperref[props:argmax]{\Cref{prop:properties}\ref{props:argmax}}. The resulting differentiable version of IHT, which we refer to as \emph{Soft-IHT}, is presented below.	
	
	\begin{algorithm}[htbp]
		\textbf{Input:} desired sparsity $ k $, mixing matrix $ A \in \bC^{m \times N} $ with $ \ell^2 $-normalized columns, observation vector $ y \in \bC^N $, initial signal $ \tilde{x}^{(0)} \in \bC^N $, step size $ \eta > 0 $, number of iterations $\bar{n}$.\\
		\vspace{-4mm}
		\begin{tcolorbox}[boxsep=3pt,
			left=1pt,
			right=1pt, colback=white, colframe=black, boxrule=0.5mm, opacityframe=0.3]
			Let $ \tilde{Q}^{(0)} = [\ ] $.\\
			For $ n = 0, \dots, \bar{n} - 1 $ repeat:
			\begin{equation*}
				\begin{cases}
					\tilde{u}^{(n + 1)} = \tilde{x}^{(n)} + \eta A^*(y - A\tilde{x}^{(n)}), \\
					\tilde{Q}^{(n + 1)} = \text{softsort}_\tau(\tilde{v}^{(n + 1)}),\quad \tilde{v}^{(n + 1)} = \left|u^{(n + 1)}\right|\\
					\tilde{q}^{(n + 1)} = \sum_{i = 1}^k\tilde{Q}^{(n + 1)}(i, :), \\
					\tilde{x}^{(n + 1)} = \tilde{q}^{(n + 1)} \odot \tilde{u}^{(n + 1)},
				\end{cases} \tag{Soft-IHT}
			\end{equation*}
		\end{tcolorbox}
		\textbf{Output:} Approximately $ k $-sparse vector $ \hat{x} = \tilde{x}^{(\bar{n})} \in \bC^N $.
		\caption{Soft-IHT.}
		\label{alg:Soft-IHT}
	\end{algorithm}
	
	Just like $ q^{(n + 1)} $, $ \tilde{q}^{(n + 1)} $ also represents a multiplication operator that acts as a mask whose dominant entries correspond to top values of the signal. The performance of Soft-IHT relative to IHT depends on how well this mask approximates the true mask. Similar to the OMP case, the following theorem quantifies this approximation in terms of a condition on $ \tau $, the temperature parameter of softsort.
	\begin{theorem}[Soft-IHT is a good approximation for IHT]
		\label{thm:iht_main}
		For $n \in \bN_0$ fixed, let $ x^{(n)} \in \bC^{N} $, $v^{(n)} \in \bR^{N}$ and $ Q^{(n)} \in \bR^{N \times N} $ be \graysout{sequences} generated by \Cref{alg:pIHT}, and $ \tilde{x}^{(n)} \in \bC^{N} $, $\tilde{v}^{(n)} \in \bR^{N}$ and $ \tilde{Q}^{(n)} \in \bR^{N \times N} $ be \graysout{sequences} generated by \Cref{alg:Soft-IHT}. Further, assume that $ x^{(0)} = \tilde{x}^{(0)} $. Then, the following hold:
		\begin{enumerate}[leftmargin=1.2cm, label=(\roman*)]
			\item Asymptotic convergence: As $ \tau \to 0 $, we have $ \tilde{Q}^{(n)} \to Q^{(n)} $ and $ \tilde{x}^{(n)} \to x^{(n)} $.
			\item Non-asymptotic convergence: Choose
			\[0 < \epsilon < g^{(1:n)}/2L,\]
			with $ L = \|I - \eta A^*A\| $ and $g^{(1:n)}$ defined as 
			\begin{equation}\label{eq:iht_gap}
				g^{(1:n)} := \min_{i \in [n]}g^{(i)},\quad g^{(i)} := \min_{j,j' \in [N], \, j  \neq j'}\left|v_{j}^{(i)} - v_{j'}^{(i)}\right|,
			\end{equation}
			representing IHT's ``global min-gap" and ``local min-gap", respectively.
			If $ \tau $, the temperature parameter of $\text{softsort}$, satisfies the condition
			\[\tau \leq \left(g^{(1:n)} - 2L\epsilon\right)/\log(C^{(n)}/\epsilon),\]
			where
			\begin{equation}
				\label{eq:tau-iht}
				C^{(n)} = 2sN\frac{(sL)^n - 1}{sL - 1}\left(\left(|1-\eta| + \eta s\mu(A)\right)\max_{0 \leq k \leq n - 1}\|x^{(k)}\|_\infty + \eta\|y\|\right),
			\end{equation}
			and $ \mu(A) $ is the coherence parameter of the matrix $ A $, then
			\[\max_{i \in [0:n]}\left\|x^{(i)} - \tilde{x}^{(i)}\right\| \leq \epsilon.\]
		\end{enumerate}
	\end{theorem}
	\begin{proof}[Proof (sketch)]
		The proof of this theorem is derived through a perturbation analysis of a single IHT iteration, ensuring that a single iteration of Soft-IHT follows IHT in maximal indices passed through the next iteration and approximation error (\Cref{lemma:iht_single_iteration}). This result is then employed in an induction to extend the result to all previous iterations down to the first iteration. For the full proof, see \Cref{appendix:iht}.
	\end{proof}
	Similar to the conclusions drawn in the OMP case, achieving smaller gaps (inherent to the model) and a higher approximation accuracy requires smaller values of $ \tau $. Moreover, \Cref{remark:gaps,remark:recovery_gurantee} remain valid in the IHT case as well.
	
	\section{Greedy networks}
	\label{sec:greedy_networks}
	\Cref{alg:Soft-OMP,alg:Soft-IHT} yield good approximations of \Cref{alg:OMP,alg:IHT}, respectively, achieving an accuracy level prescribed by \Cref{thm:omp_main,thm:iht_main} and regulated by $ \tau $. However, this approximation comes with some accuracy loss due to replacing the exact argsort operator in \Cref{eq:sorting} with the approximate softsort operator in \Cref{eq:softsort}. This approximation error can be tolerated as it facilitates the implementation of Soft-$\ast$ algorithms within neural networks. 
	
	In this section, we introduce these neural networks that we call \emph{OMP-Net} and \emph{IHT-Net}, placing them under the umbrella of \emph{greedy networks}, and describe their training procedure. In \Cref{sec:numerics}, we will provide numerical evidence that these networks can not only compensate for the introduced approximation error through training, but even outperform the original algorithms.
	
	\paragraph{OMP-Net and IHT-Net}
	Our goal is to approximate the reconstruction map associated with problem \eqref{eq:SR_model}, i.e., $ \Delta: \bC^m \times \bC^{m \times N} \ni (y, A) \mapsto x = \Delta(y, A) \in \bC^N $ (or simply $\Delta(y)$ as $A$ is fixed throughout), using a neural network built upon Soft-OMP and Soft-IHT. Although a variety of choices exist for trainable parameters (see \Cref{sec:conclusion}), we specifically take the weights as trainable parameters, embedding them into the greedy criterion of Soft-OMP and the hard-thresholding operator of Soft-IHT, as explained in \Cref{sec:preliminaries}. This aims to uncover the latent structure within the signal and naturally sets forth an application of our construction in the context of weighted sparse recovery. Thus, constructing $L$ layers of Soft-$\ast$ algorithms leads to trainable parameters $\Theta = (w^{(l)})_{l = 1}^L$. Given the dataset $\mathcal{D} = \mathcal{S}\sqcup\mathcal{S}'$, where the training data $ \mathcal{S} = \{(x^{(i)}, y^{(i)})\}_{i = 1}^{N_{\text{tr}}} $ consists of $N_\text{tr}$ instances of input-output pairs for training, and correspondingly the validation data $ \mathcal{S}' = \{(x^{(i)}, y^{(i)})\}_{i = 1}^{N_{\text{val}}} $, we train the neural network $ y \in \bC^m \mapsto \mathcal{NN}_{\mathcal{S}, \Theta}(y) \in \bC^N $. The network $\mathcal{NN}_{\mathcal{S}, \Theta}(\cdot)$ uses the training data $\mathcal{S}$ to learn a data-driven approximation of the mapping $\Delta$ via gradient-based optimization. The parameters are iteratively updated over $N_\text{epoch}$ epochs, generating the sequence $\Theta^{(j)},\ j = 0, \dots, N_{\text{epoch}}$, where $\Theta^{(0)}$ represents the network's initial state. The optimization minimizes the \emph{squared-loss}
	\[\text{Loss}_{sq}(x, x') = \|x - x'\|^2 \in [0, +\infty),\]
	yielding the training error
	\[\text{MSE}_{\text{tr}}^{(j)}(\mathcal{S}) = \frac{1}{N_{\text{tr}}}\sum_{i = 1}^{N_{\text{tr}}}\text{Loss}_{sq}\left(\mathcal{NN}_{\mathcal{S}, \Theta^{(j)}}\left(y^{(i)}\right), x^{(i)}\right),\ j = 0, \dots, N_{\text{epoch}},\]
	computed over the training data $\mathcal{S}$, with a similar validation error $\text{MSE}_{\text{val}}^{(j)}(\mathcal{S}')$ for the validation data $\mathcal{S}'$.
	
	\section{Numerical experiments}
	\label{sec:numerics}
	We now present numerical experiments to validate our construction. The source code of our numerical experiments can be found on the GitHub repository \url{https://github.com/sina-taheri/Deep_Greedy_Unfolding}. We first demonstrate that Soft-OMP and Soft-IHT serve as viable alternatives to OMP and IHT, capable of approximating the original algorithms with the desired accuracy depending on $\tau$, the temperature parameter of softsort. We then illustrate their effectiveness in the scenario of severe undersampling, i.e., when $m$ is so small that conventional compressed sensing struggles to reach a good recovery accuracy. By training neural networks based on these algorithms, we surpass the performance of conventional compressed sensing. Since all our experiments rely on a random compressed sensing setup with Gaussian or Fourier measurements, we first introduce this setup.
	\paragraph{Random compressed sensing}
	We generate an $s$-sparse signal $x \in \bR^N$ by selecting $s \ll N$ random indices $i \in S \subset [N]$ without replacement from a discrete uniform distribution and assigning to each nonzero entry $x_i$ a value from a standard random Gaussian distribution, i.e., $x_i \sim \mathcal{N}(0, 1)$. The signal $x$ is then passed through the measurement matrix $A = A'/\sqrt{m} \in \mathbb{F}^{m \times N}$, where $A'$ is either (1) a real-valued Gaussian matrix ($\mathbb{F} = \bR$) with elements identically and independently distributed according to the $\mathcal{N}(0, 1)$ distribution, or (2) a partial Fourier matrix ($\mathbb{F} = \bC$) given by $A' = PF$, where $P \in \bR^{m \times N}$ is a row-selector matrix and $F \in \bC^{m \times N}$ is the Fourier matrix with elements $F_{kl} = \exp(-i2\pi kl/N)$ for $k = -N/2 + 1, \dots, N/2$ and $l = 0, \dots, N - 1$. Our goal is to recover $x$ from the measurements \eqref{eq:SR_model}, where $e$ is a random Gaussian noise vector with $e_i \sim \mathcal{N}(0, \sigma^2/m)$, real or complex-valued in the case of Gaussian or Fourier matrix respectively. We achieve this by running $l$ iterations of OMP and IHT, or their differentiable counterparts, tackling the sparse recovery problem \eqref{eq:SR_problem}. This yields $x^{(l)} = \mathcal{A}(y, l)$, where $\mathcal{A}$ corresponds to OMP or IHT (pOMP or pIHT), or $\tilde{x}^{(l)}(\tau) = \tilde{\mathcal{A}}(y, l,\tau)$ for Soft-OMP or Soft-IHT.
	\subsection{Experiment I: Validation of \Cref{thm:omp_main,thm:iht_main}}\label{subsec:experiment_i}
	We numerically verify that Soft-OMP and Soft-IHT serve as accurate approximations of OMP and IHT, experimentally validating \Cref{thm:omp_main,thm:iht_main}. Our compressed sensing setup employs Gaussian measurements with the following parameters:
	\[N = 400,\quad m = 200,\quad s = 15,\quad \sigma = 10^{-3},\quad \eta = 0.6 \; \text{(for IHT)}.\]
	We reconstruct the signals $x^{(n)}$ and $\tilde{x}^{(n)}(\tau)$ while varying $\tau$ and compute the relative $\ell^2$-error 
	\[E^{(n)}(\tau) = \frac{\|x^{(n)} - \tilde{x}^{(n)}(\tau)\|}{\|x^{(n)}\|}.\]
	This error is plotted as a function of $\tau$ for different iteration counts: $n \in \{5, 15, 30\}$ for \mbox{(Soft-)OMP} and $n \in \{1, 15, 45\}$ for (Soft-)IHT, respectively. The procedure is repeated $N_\text{trial}$ times, and the results are visualized as shaded plots in \Cref{fig:shaded_plot}. The solid curves $(\tau, 10^{\mu_\tau^{(n)}})$ are bounded by $(\tau, 10^{\mu_\tau^{(n)} + \rho_\tau^{(n)}})$ and $(\tau, 10^{\mu_\tau^{(n)} - \rho_\tau^{(n)}})$ where the logarithmic mean $\mu_\tau^{(n)}$ and standard deviation $\rho_\tau^{(n)}$ are given by
	\[\mu_\tau^{(n)} = \frac{1}{N_\text{trial}}\sum_{i = 1}^{N_\text{trial}}\left(\log\left(E^{(n)}(\tau)\right)\right)_i\text{and}\ \rho_\tau^{(n)} = \sqrt{\frac{1}{N_\text{trial} - 1}\sum_{i = 1}^{N_\text{trial}}\left[\left(\log(E^{(n)}(\tau))\right)_i - \mu_\tau^{(n)}\right]^2}.\]
	To prevent numerical issues with logarithms, we compute $E^{(n)}(\tau) + \epsilon$ in practice, where $\epsilon=10^{-12}$.
	\begin{figure}[t!]
		\includegraphics[width = 0.49\textwidth]{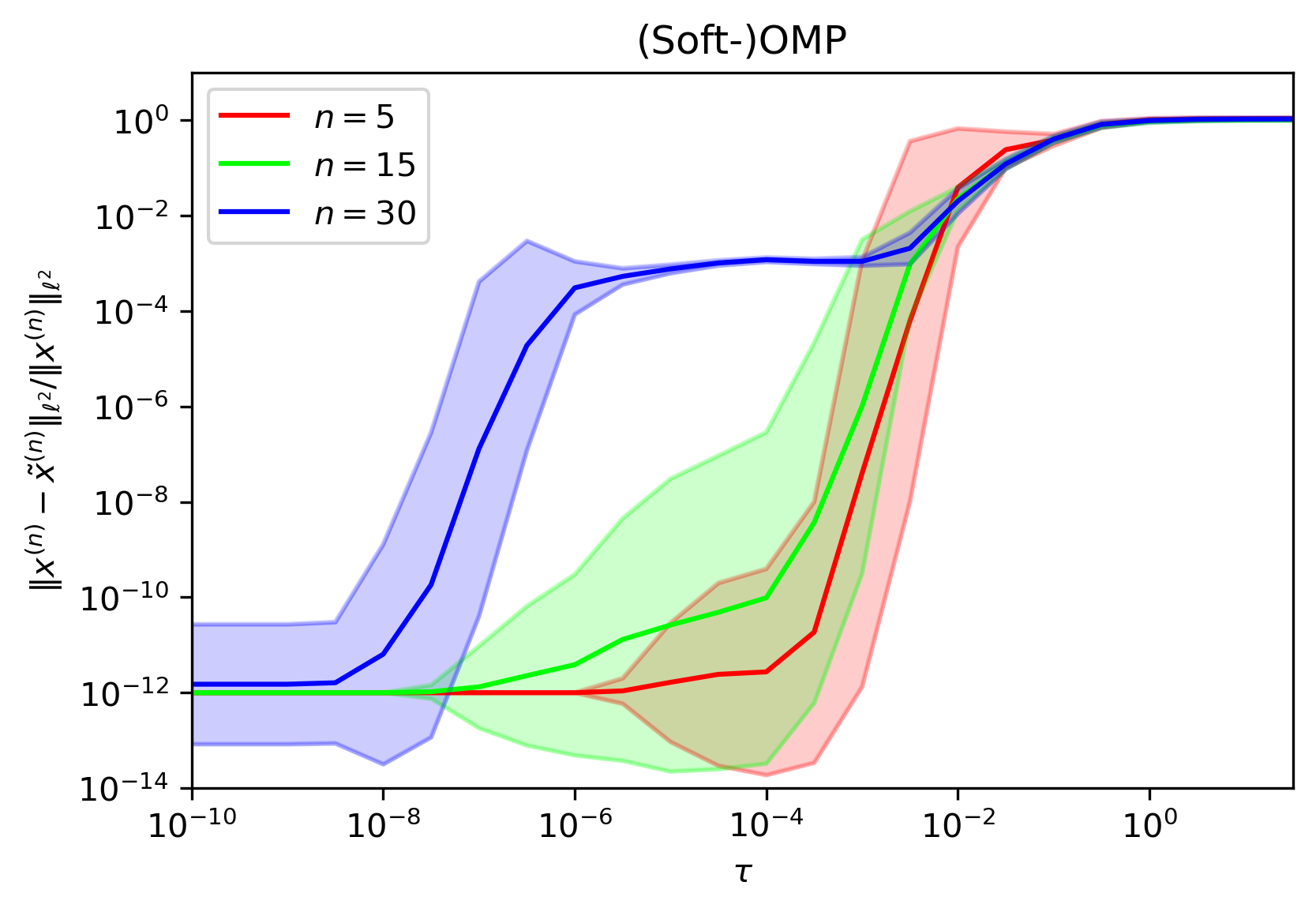}%
		\includegraphics[width = 0.49\textwidth]{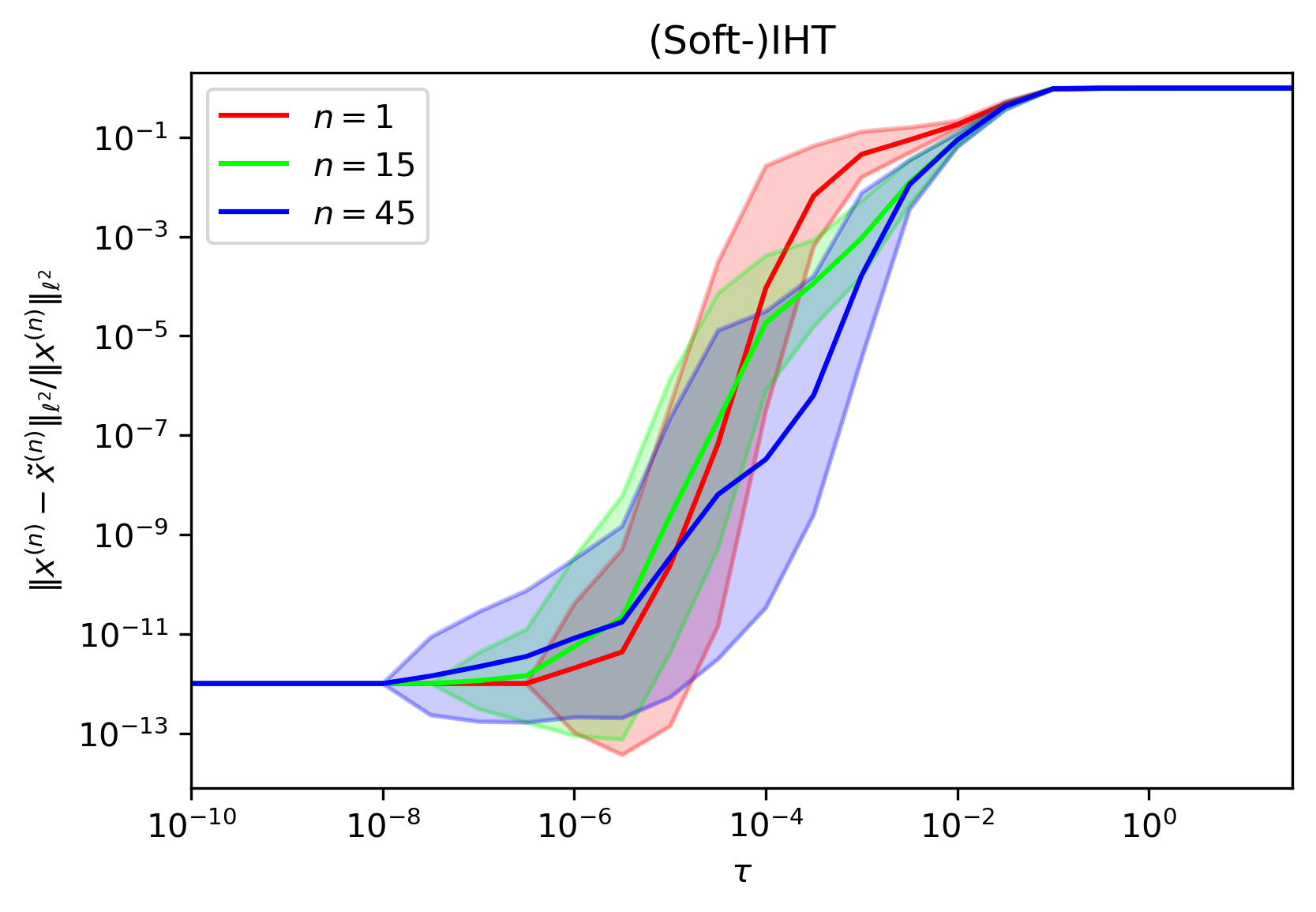}
		\caption{Relative $\ell^2$-error as a function of $\tau$ (see \Cref{subsec:experiment_i}). Recovery accuracy of (Soft-)OMP on the left and (Soft-)IHT on the right for various iteration counts.}
		\label{fig:shaded_plot}
	\end{figure}
	\Cref{fig:shaded_plot} confirms that for sufficiently small $\tau$, Soft-OMP and Soft-IHT converge to their original counterparts. As expected, achieving this convergence for larger iteration counts requires smaller values of $\tau$, consistent with \Cref{thm:omp_main,thm:iht_main}. However, the precise rate of convergence varies across experiments, as the optimal $\tau$ depends on the problem-specific global min-gap, which tend to diminish with further iterations. On this point, one must be conscious of the role of the measurement operator on the choice of $\tau$, depending on, e.g., the gaps and measurement matrix's spectral norm (recall conditions \eqref{eq:tau-omp} and \eqref{eq:tau-iht}). In our experimental setting, however, switching to Fourier measurements did not lead to any significant changes in the results associated with \Cref{fig:shaded_plot}. See \Cref{sec:impact_A_on_tau} for further details.
	
	\subsection{Experiment II: Weighted recovery in the undersampled regime}\label{subsec:experiment_ii}
	We now turn to the core motivation behind Soft-OMP and Soft-IHT: their unfolding into neural networks, corresponding to the trainable architectures OMP-Net and IHT-Net, respectively. In experiments of this section for IHT, the step size $\eta$ is fixed (see, e.g., \cite{adcock2022sparse}). We construct a dataset using a sparse random model with a partial Fourier matrix $A \in \bC^{m \times N}$, deliberately setting $m$ low enough that standard OMP and IHT fail to reconstruct the signal efficiently. To exploit the learning capabilities of these networks, we generate signals whose support $S$ lies within a larger set $T \supset S$ of size $|T| = ks,\ k \in \mathbb{N}$, uniformly distributed over $[N]$. This motivates the definition of oracle weights $w^{\text{oracle}} \in \bR^N$
	\[w_j^{\text{oracle}} = \begin{cases}
		1 & j \in T,\\
		0 & \text{otherwise}.
	\end{cases}\]
	OMP-Net and IHT-Net do not have access to the oracle information. Instead, they are initialized with uniform weights $w^\text{NN} = \mathbbm{1}$, and are expected to infer the latent structure during training.
	
	We generate the dataset $\mathcal{D}$ with parameters
	\begin{align*}
		&N = 256,\quad s = 10,\quad m = 22\ \text{for\ OMP\ and}\ 36\ \text{for\ IHT},\quad \eta = 0.5\ \text{(for IHT)}\\
		&\sigma = 10^{-3},\quad k = 2,\quad N_\text{train} = 1024,\quad N_\text{val} = 512.
	\end{align*}
	We construct OMP-Net and IHT-Net with $L$ layers and employ weight-sharing across layers, i.e., $w^{(l)} = w^\text{NN} \in \bR^N$ for all $l = 1, \dots, L$. We train them using RMSprop for $N_\text{epoch}$ epochs with an empirically chosen learning rate $lr$ and gradient norm clipping with maximum gradient norm set to 1. We further employ checkpointing to improve training: every $N_\text{checkpt}$ epochs, the network parameters are saved, and the final model corresponds to the checkpoint with the best performance on the training data $\mathcal{S}$. The network parameters are:
	\[L = 10\ \text{for\ OMP\ and}\ 30\ \text{for\ IHT},\quad N_\text{epoch} = 1000,\quad N_\text{checkpt} = 10,\quad lr = 10^{-2},\quad \tau = 10^{-3}.\]
	As mentioned earlier, careful selection of $\tau$ is essential. While \Cref{thm:omp_main,thm:iht_main} suggest small values, excessively small $\tau$, can lead to vanishing gradients and non-trainability, while overly large $\tau$ sacrifices accuracy to an extent that would make it very difficult to compensate for by training. Consequently, one might select $\tau$ as ``the smallest trainable value", found through hyperparameter tuning. We have also empirically observed that a value of $\tau$ that achieves such a good balance, lies within the transition phase of \Cref{fig:shaded_plot}, suggesting a strategy to select the appropriate value in view of the aforementioned trade-off.
	
	The results, presented in \Cref{fig:nn_training}, compare OMP and IHT (first and second columns, respectively). The top row shows the evolution of $\text{MSE}_\text{tr}$ and $\text{MSE}_\text{val}$, truncated at the best checkpoint. Also, the red and green lines represent the MSE-error of the algorithm and the network loaded with the weights at the best checkpoint, respectively, evaluated on the validation set $\mathcal{S}'$. For IHT number of iterations is always equal to $L$ and $\eta = 0.5$. The stem plots in the second and third rows compare oracle and learned weights, illustrating the networks' ability to detect the set $T$ from data. Since mean-squared error is sensitive to outliers, we complement our analysis with boxplots of the relative $\ell^2$-error (fourth row), capturing the probabilistic nature of the experiment and confirming near noise-level signal recovery. Notably, OMP-Net and IHT-Net are able to outperform their classical counterpart by orders of magnitude.
	\begin{figure}[t!]
		\includegraphics[width = 0.49\textwidth]{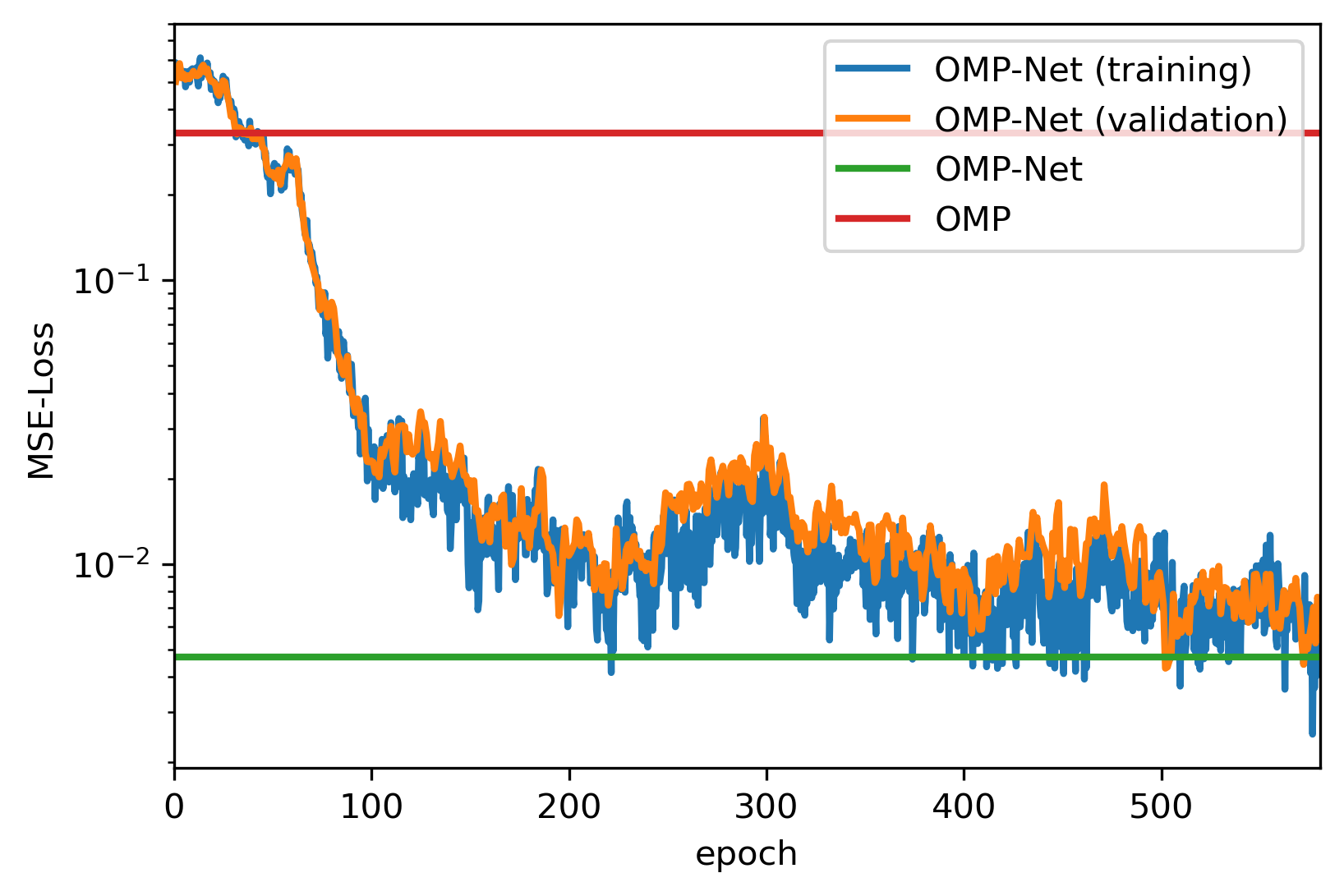}%
		\includegraphics[width = 0.49\textwidth]{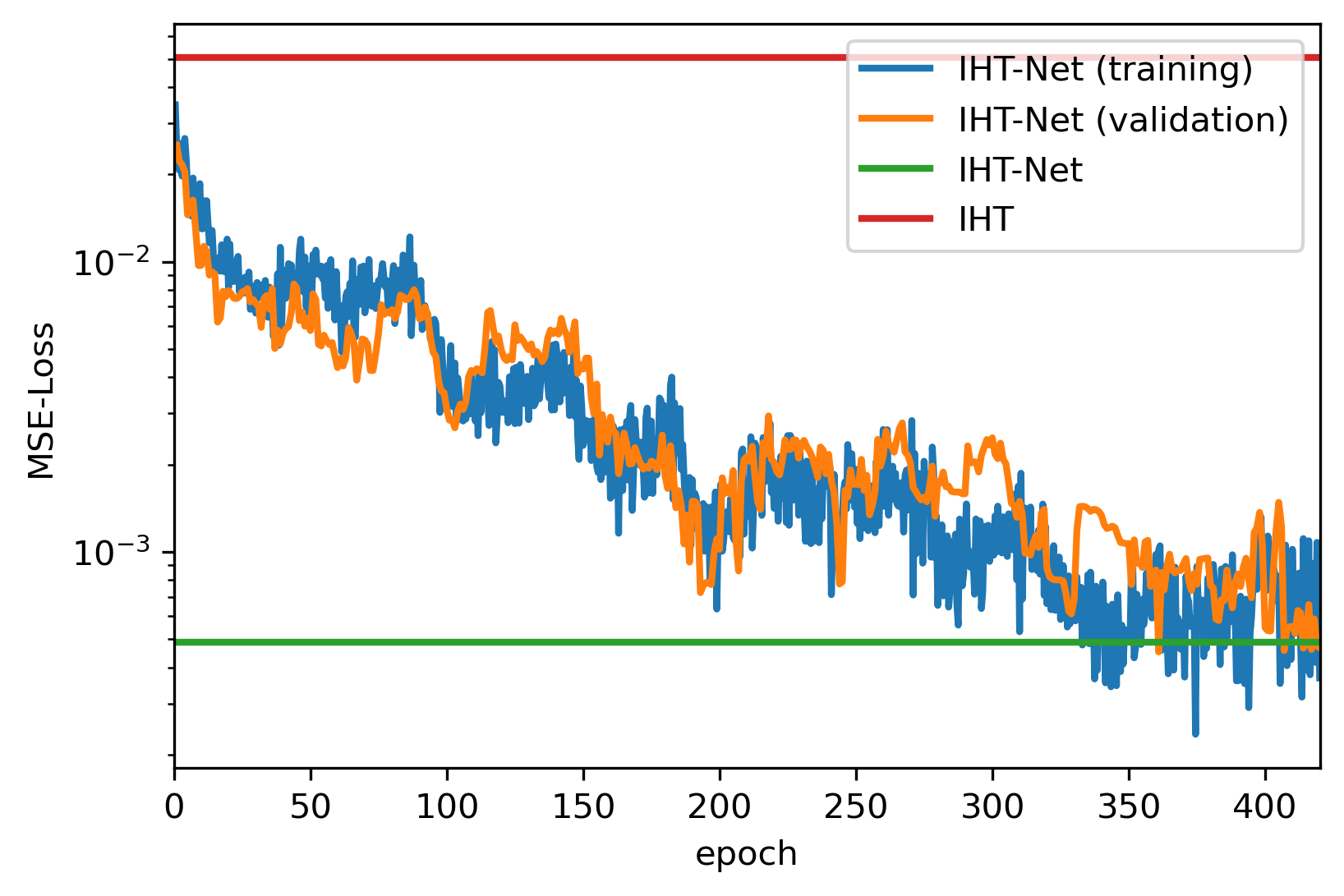}
		
		\includegraphics[width = 0.49\textwidth]{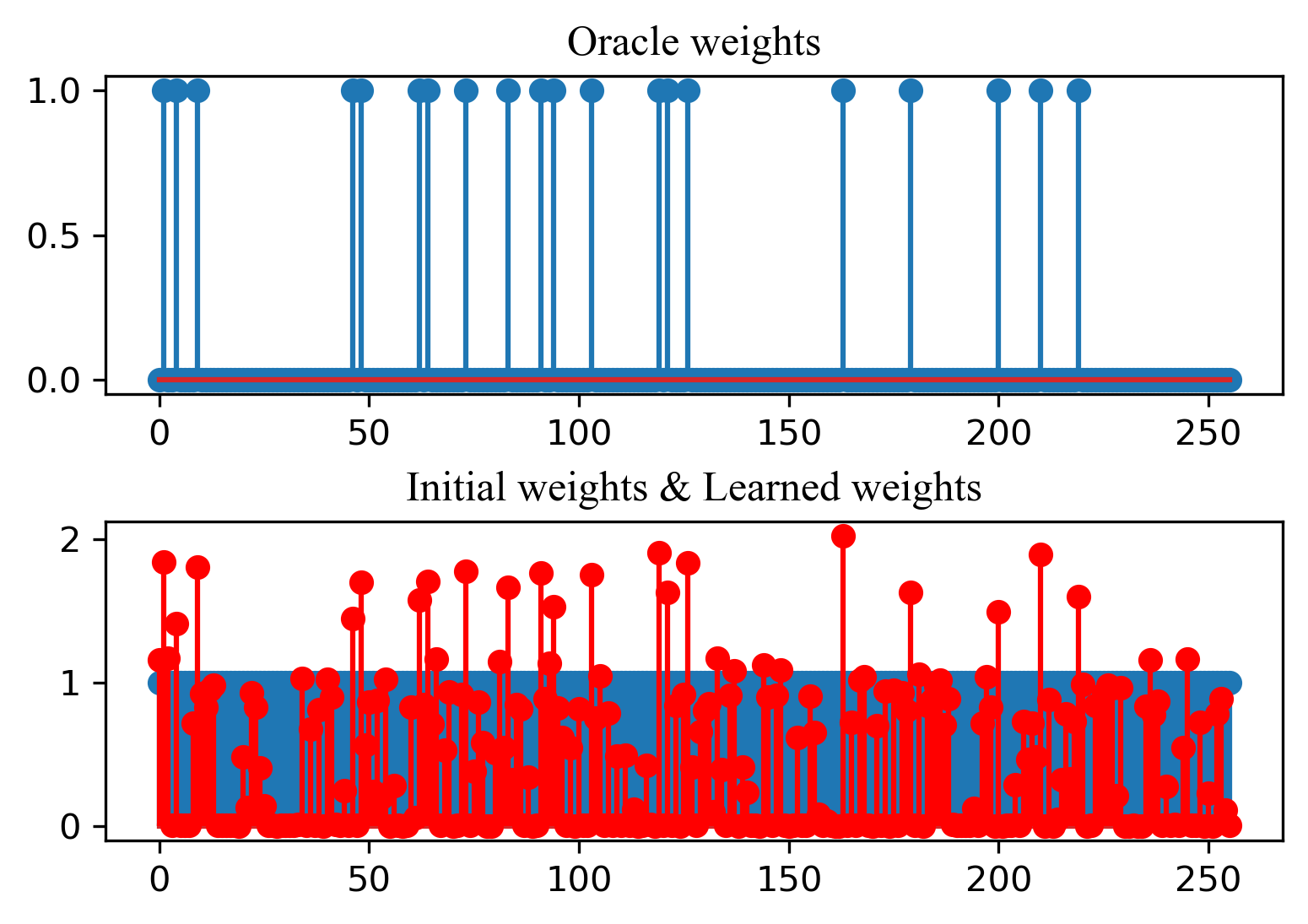}%
		\includegraphics[width = 0.49\textwidth]{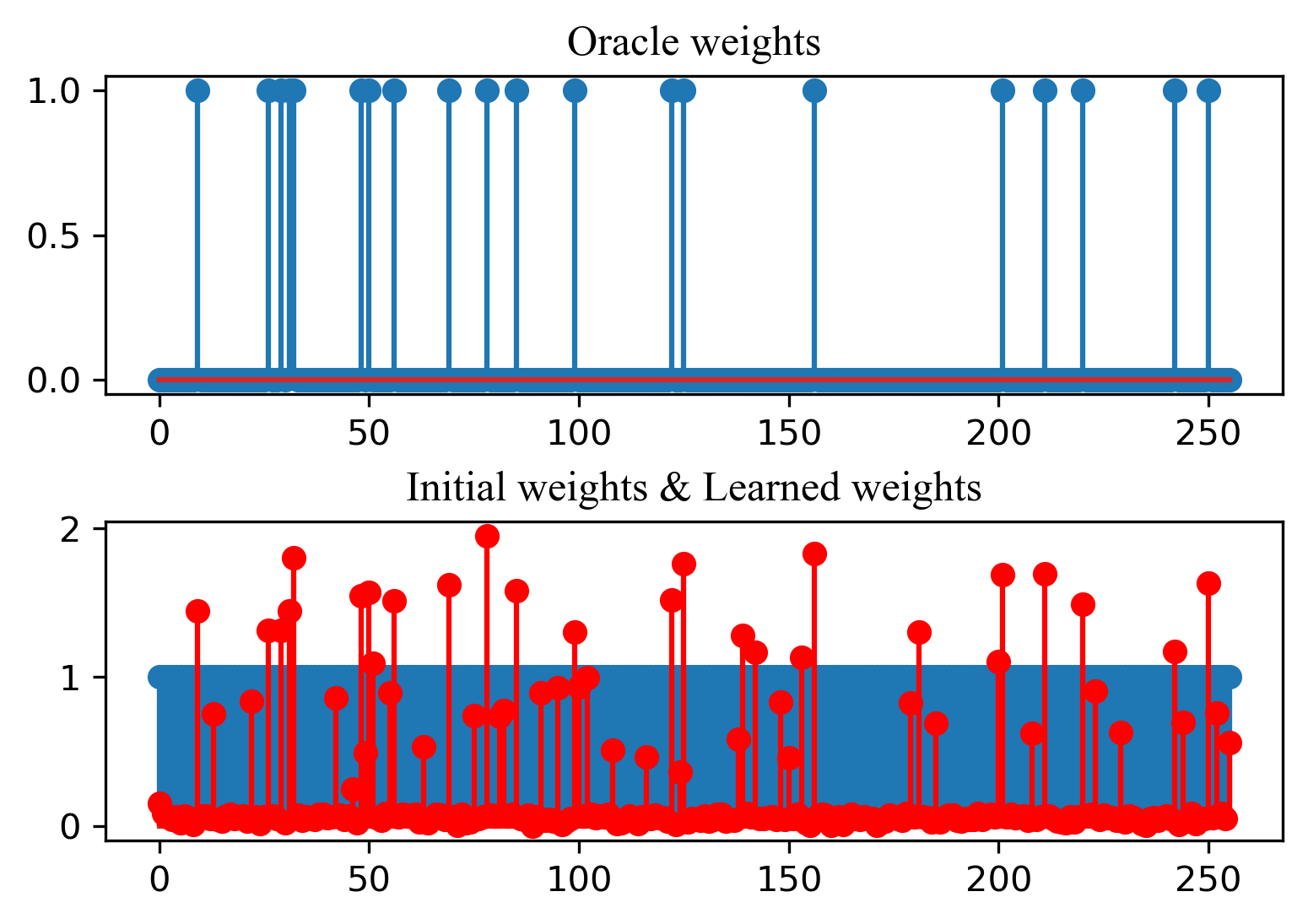}
		
		\includegraphics[width = 0.49\textwidth]{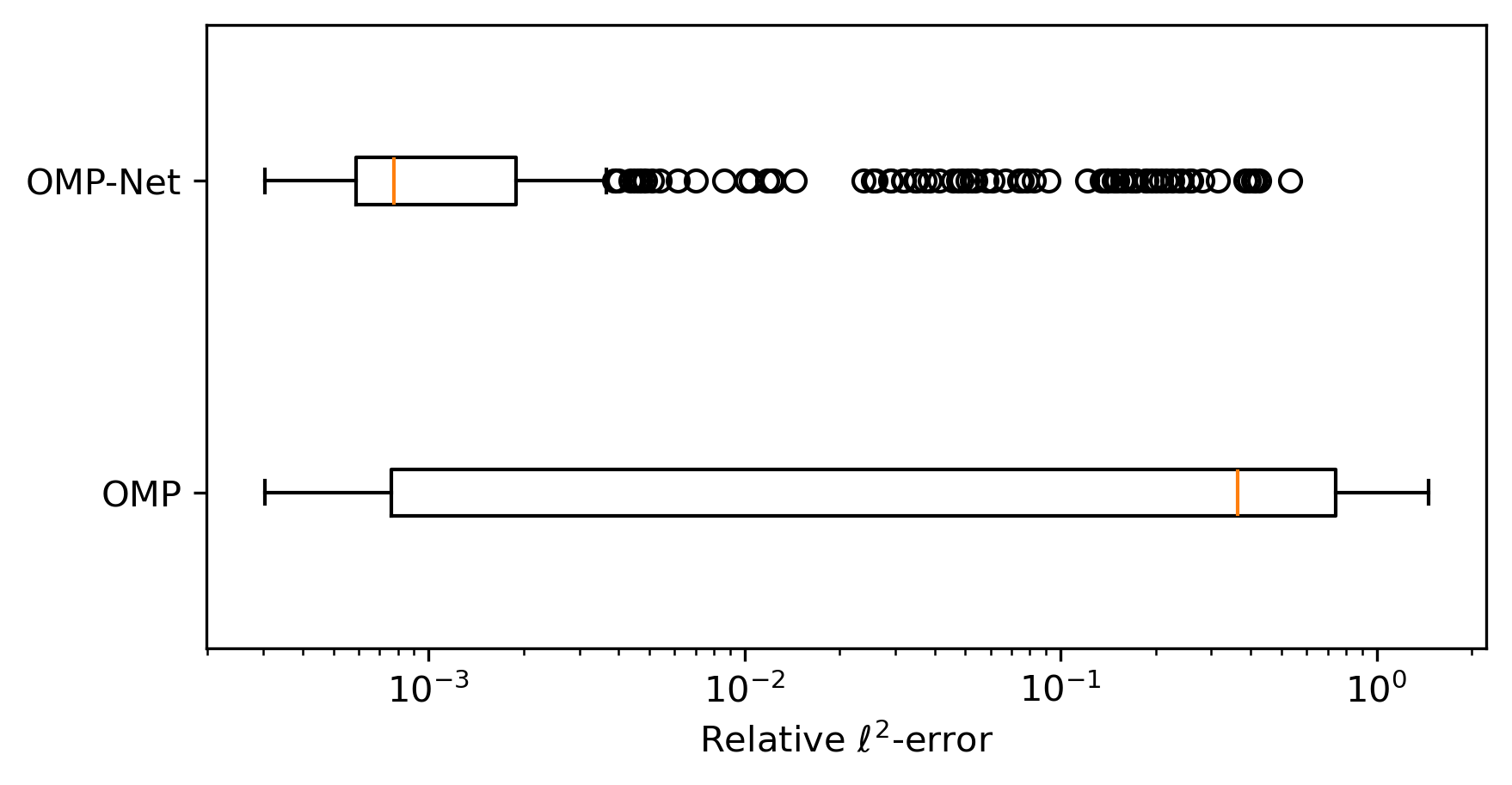}%
		\includegraphics[width = 0.49\textwidth]{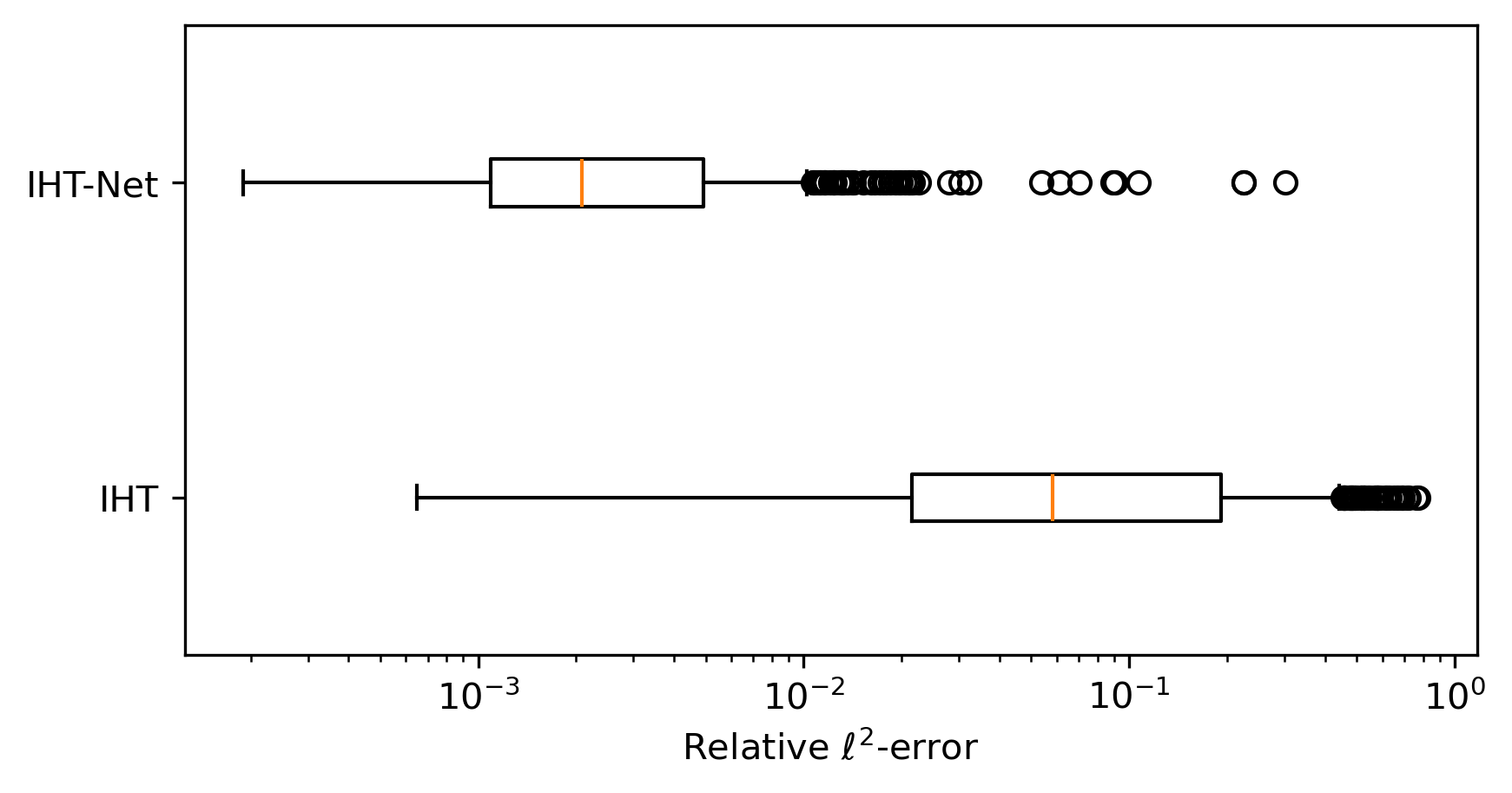}
		\caption{From top to bottom: MSE-Loss, oracle weights, learned weights and relative $\ell^2$-error boxplots for (Soft-)OMP on the left and (Soft-)IHT on the right column (see \Cref{subsec:experiment_ii} for futher details).}
		\label{fig:nn_training}
	\end{figure}
	\section{Conclusions and future research}\label{sec:conclusion}
	In this manuscript, we directly addressed the non-differentiability of the argsort operator in greedy sparse recovery algorithms, such as OMP and IHT, and proposed  differentiable approximations, Soft-OMP and Soft-IHT, which were effectively unrolled into neural networks. We also established rigorous theoretical error bounds for Soft-OMP and Soft-IHT in relation to their original counterparts.
	
	There are several potential avenues for future research.
	First, softsort is not the only approach for implementing the argsort operator in a differentiable manner within neural networks. The development of neural-network compatible implementations of (arg)sort has long been an area of active interest. However, differentiable realizations of such operators have only emerged in the last decade. Notable examples include Neuralsort \cite{grover2019stochastic}, which shares a significant resemblance to softsort, optimal transport-based methods using the Sinkhorn algorithm \cite{cuturi2019differentiable, mena2018learning}, optimization-based approaches \cite{blondel2020fast,sander2023fast}, and the Gumble trick \cite{jang2016categorical}. Among these, we chose softsort primarily due to its simple formulation and interpretability. That said, a comprehensive comparison of greedy networks based on the various differentiable sorting methods available in the literature remains a path to be pursued.
	
	Second, in this work we unrolled OMP and IHT as prominent representatives of greedy algorithms. Other well-known candidates worth mentioning include, but are not limited to, CoSaMP \cite{needell2009cosamp}, subspace pursuit \cite{dai2009subspace} and greedy weighted-LASSO algorithms \cite{mohammad2025greedy}.
	
	Third, in our implementation we use the most natural architecture directly imposed by the model. However, generalizations to other architectures, such as Recurrent Neural Networks (RNNs) to capture the recurrent nature of algorithms or Convolutional Neural Networks (CNNs) to reduce the number of parameters, warrant further investigation. Additionally, several interesting extensions could be explored, including experiments with varying numbers of layers, employing distinct parameters across layers, realization of stopping criterion, and learning additional parameters (e.g., the step size $\eta$, the full matrix $A^*$, $I - A^*A$, etc.). In particular, choosing the optimal value of $\tau$ is a practical challenge that deserves further investigation. Beyond the recipes mentioned in \Cref{sec:numerics}, one may further decide to make $\tau$ a trainable parameter of the unrolled network. This may require additional care during the training process, for example by using different learning rates or imposing value constraints to ensure that $\tau$ remains within an admissible range.
	
	\section*{Acknowledgements} SM acknowledges support from MITACS Globalink research award. SB was partially supported by the Natural Sciences and Engineering Research Council of Canada (NSERC) through grant RGPIN-2020-06766 and the Fonds de Recherche du Qu\'ebec Nature et Technologies (FRQNT) through grants 313276 and 359708. MC acknowledges support from INT/UCAM Early Career grant LEAG/929.G101121. The authors thank the anonymous reviewers whose comments highlighted the issue in the first version of \Cref{thm:omp_main,thm:iht_main} and led to the improvement of these results.
	
    \newpage

    \appendix
	\section{Proof of \Cref{prop:continuity}}\label{appendix:continuity}
	\begin{proof}
		Continuity of softsort comes from the fact that softsort is a composition of continuous functions. Let $ u^* = \sort(u) $ and $ v^* = \sort(v) $, for any $ u, v \in \bR^N $. Then knowing that $ \softmax $ is a $1$-Lipschitz continuous function \cite{gao2017properties}, for all $ i \in [N] $ we can write
		\begin{align*}
			&\left\|\softsort(u)(i, :) - \softsort(v)(i, :)\right\|\\
			&=\left\|\softmax\left(\frac{-|u^*_i\mathds{1} - u|}{\tau}\right) - \softmax\left(\frac{-|v^*_i\mathds{1} - v|}{\tau}\right)\right\| \leq \frac{1}{\tau}\left\|\left|u^*_i\mathds{1} - u\right| - \left|v^*_i\mathds{1} - v\right|\right\| \\ 
			&\leq \frac{1}{\tau}\left\|\left(u^*_i - v^*_i\right)\mathds{1} - \left(u - v\right)\right\| \leq \frac{1}{\tau}\left(\left\|\left(u^*_i - v^*_i\right)\mathds{1}\right\| + \left\|u - v\right\|\right) \\
			&\leq \frac{1}{\tau}\left(\sqrt{N}\|u^* - v^*\|_\infty + \|u - v\|\right) \overset{(*)}{\leq} \frac{1}{\tau}\left(\sqrt{N}\|u - v\|_\infty + \|u - v\|\right) \\
			&\leq \frac{\sqrt{N} + 1}{\tau}\|u - v\|,
		\end{align*}
		where $(*)$ holds due to the inequality $\|u^* - v^*\|_\infty \leq \|u - v\|_\infty$. For a proof of the latter when $u, v \in \bR_+^N$, see \cite[Lemma 2.12]{foucart2013mathematical}. The generic case $u, v \in \bR^N$ can be established by applying \cite[Lemma 2.12]{foucart2013mathematical} to $\tilde{u} = u - h$ and $\tilde{v} = v - h$ where $h := \min\{u_1, \dots, u_N, v_1, \dots, v_N\}$, observing that additive shifts do not impact the sorting operation.
	\end{proof}
	\section{Proof of {\Cref{thm:omp_main}}}\label{appendix:omp}
	At the heart of our proof lies the sensitivity analysis of least-squares, which quantifies how perturbations in the decomposition space affect the least-squares solution $ w $. The following lemma provides a precise bound for this deviation in terms of $ \kappa(B) $, the $\ell^2$-norm condition number of $ B $.
	\begin{lemma}[\S Theorem 18.1 \cite{trefethen2022numerical}]
		\label{th:Trefethen}
		Let $ y \in \bC^m $ and $ A \in \bC^{m \times n} $ of full rank be fixed. The least squares problem, i.e.,
		\[x = \argmin{z \in \mathbb{C}^n}{\|y - Az\|^2},\]
		has the following $ \ell^2 $-norm relative condition number describing the sensitivity of $ x $ with respect to $ A $:
		\[\sup_{\delta A}\left(\frac{\|\delta x\|}{\|x\|}/\frac{\|\delta A\|}{\|A\|}\right) \leq \kappa(A) + \frac{\kappa(A)^2\tan \theta}{\eta}.\]
		Here $ \delta x \in \bC^n $ and $ \delta A \in \bC^{m \times n} $ represent perturbations in $ x $ and $ A $ respectively, $ 1 \leq \kappa(A) = \|A\|\|A^\dagger\| $ is the $\ell^2$-norm condition number of $ A $, and
		\[0 \leq \theta = \cos^{-1}\frac{\|Ax\|}{\|y\|} \leq \pi/2,\quad 1 \leq \eta = \frac{\|A\|\|x\|}{\|Ax\|} \leq \kappa(A).\]
	\end{lemma}
	Assuming $ \argmax{}{(\Pi^{(n)})} = \argmax{}{(\tilde{\Pi}^{(n)})} $, the next lemma bounds $\|\Pi^{(n)} - \tilde{\Pi}^{(n)}\|$ and $\|\tilde{\Pi}^{(n)}\|$.
	
	\begin{lemma}
		\label{lemma:frob_bounds}
		Let $ \Pi^{(n)} $ and $ \tilde{\Pi}^{(n)} $ be sequences of projection matrices generated by \Cref{alg:pOMP,alg:Soft-OMP} respectively and assume $ \argmax{}{v^{(i)}} = \argmax{}{\tilde{v}^{(i)}}$ for all $i \in [n] $. Then we have
		\[\|\tilde{\Pi}^{(n)}\|_F \leq \sqrt{n},\quad
		\|\Pi^{(n)} - \tilde{\Pi}^{(n)}\|_F \leq \sqrt{2n}(N - 1)e^{-\tilde{g}^{(1:n)}/\tau},\]
		where $ \tilde{g}^{(1:n)} $ represents Soft-OMP's ``global min-gap", defined as
		\begin{equation}\label{eq:soft-omp-gap}
			\begin{aligned}
				\tilde{g}^{(1:n)} & := \min_{i \in [n]}\tilde{g}^{(i)}, \\
				\tilde{g}^{(i)}  & := \min_{j \in [N], j\neq \tilde{j}^{(i)}}\left|\tilde{v}_j^{(i)} - \tilde{v}_{\tilde{j}^{(i)}}^{(i)}\right|\ \text{with}\ \tilde{j}^{(i)} := \argmax{}{(\tilde{v}^{(i)})}.
			\end{aligned}
		\end{equation}
	\end{lemma}
	\begin{proof}
		For $ \|\tilde{\Pi}^{(n)}\|_F $, adopting the notation
		\begin{equation}\label{eq:alpha}
			\tilde{j}^{(i)} := \argmax{j \in [N]}{\tilde{v}_j^{(i)}},\quad \alpha^{(i)} := \sum_{j = 1}^{N}\exp\left(-\left|\tilde{v}^{(i)}_{\tilde{j}^{(i)}} - \tilde{v}_j^{(i)}\right|/\tau\right),
		\end{equation}
		we write
		\begin{align*}
			\|\tilde{\Pi}^{(n)}\|_F = \sqrt{\sum_{i = 1}^n \left\|\softsort{(\tilde{v}^{(i)})}(1, :)\right\|^2} &\leq \sqrt{\sum_{i = 1}^n \left\|\softsort{(\tilde{v}^{(i)})}(1, :)\right\|_1^2} = \sqrt{n},
		\end{align*}
		where we used \hyperref[props:row_affine]{\Cref{prop:properties}\ref{props:row_affine}}.

		For the second inequality, knowing that $ \argmax{}{(\Pi^{(n)})} = \argmax{}{(\tilde{\Pi}^{(n)})} $, \hyperref[props:perm_eq]{\Cref{prop:properties}\ref{props:perm_eq}} of softsort implies for all $ i \in [n] $,
		\begin{align*}
			\tilde{P}^{(i)} &= \softsort{(\tilde{v}^{(i)})} = \softsort{(\sort{(\tilde{v}^{(i)})})}P_{\argsort{(\tilde{v}^{(i)})}} \\
			&= \softsort{(\sort{(\tilde{v}^{(i)})})}P_{\argsort{(v^{(i)})}} = \softsort{(\sort{(\tilde{v}^{(i)})})}P^{(i)}.
		\end{align*}
		$ \ell^2 $-norm is permutation invariant and thanks to \hyperref[props:argmax]{\Cref{prop:properties}\ref{props:argmax}} of softsort, the location of maximum entries in each row of the matrix $ \softsort{(\sort{(\tilde{v}^{(i)})})} $ occurs on the main diagonal of the matrix, thus we write
		\allowdisplaybreaks
		\begin{align*}
			\left\|\Pi^{(n)} - \tilde{\Pi}^{(n)}\right\|_F^2 & = \sum_{i = 1}^n \left\|P^{(i)}(1,:) - \tilde{P}^{(i)}(1,:)\right\|^2\\
			&= \sum_{i = 1}^n \left\|P^{(i)}(1,:) - \left(\softsort{\left(\sort{(\tilde{v}^{(i)})}\right)}P^{(i)}\right)(1, :)\right\|^2 \\
			&= \sum_{i = 1}^n\left\|\left(I(i, :) - \softsort{\left(\sort{(\tilde{v}^{(i)})}\right)}(1, :)\right)P^{(i)}\right\|^2 \\
			&= \sum_{i = 1}^n\left\|I(i, :) - \softsort{\left(\sort{(\tilde{v}^{(i)})}\right)}(1, :)\right\|^2 \\
			&= \sum_{i = 1}^{n}\left(\left(1 - \frac{1}{\alpha^{(i)}}\right)^2 + \sum_{j = 1, j \neq i}^N\left(\exp\left({-\left|\tilde{v}_i^{(i)} - \tilde{v}_j^{(i)}\right|/\tau}\right)/\alpha^{(i)}\right)^2\right)\\
			&\leq \sum_{i = 1}^{n}\left(\left(1 - \frac{1}{\alpha^{(i)}}\right)^2 + \left(\sum_{j = 1, j \neq i}^N\exp\left(-\left|\tilde{v}_i^{(i)} - \tilde{v}_j^{(i)}\right|/\tau\right)/\alpha^{(i)}\right)^2\right) \\
			&=2\sum_{i = 1}^{n}\left(\frac{\alpha^{(i)} - 1}{\alpha^{(i)}}\right)^2 \leq 2\sum_{i = 1}^{n} \left(\alpha^{(i)} - 1\right)^2 \\
			&= 2\sum_{i = 1}^{n} \left(\sum_{j = 1, j \neq i}^N\exp\left(-\left|\tilde{v}_i^{(i)} - \tilde{v}_j^{(i)}\right|/\tau\right)\right)^2 \leq 2n(N - 1)^2e^{-2\tilde{g}^{(1:n)}/\tau},
		\end{align*}
		which concludes the proof. Here, $ \tilde{g}^{(1:n)} $ is as defined in \Cref{eq:soft-omp-gap} and $ \alpha^{(i)} \geq 1 $ is as in \Cref{eq:alpha}, where $ \tilde{j}^{(i)} = \argmax{j \in [N]}{\tilde{v}_j^{(i)}} = i $ for all $ i \in [n] $.
	\end{proof}
	
	\begin{lemma}
		\label{lemma:distance_bound}
		If $ \argmax{}{(\Pi^{(n)})} = \argmax{}{(\tilde{\Pi}^{(n)})} $ with argmax being applied row-wise, then
		\[\left\|x^{(n)} - \tilde{x}^{(n)}\right\| \leq C^{(n)}e^{-\tilde{g}^{(1:n)}/\tau},\]
		where
		\[C^{(n)} = \sqrt{2n}(N - 1)\left(\frac{\sqrt{1 - \delta_n(A)} + 2\sqrt{n}\|A\|}{1 - \delta_n(A)}\right)\|y\|,\]
		and $ \tilde{g}^{(1:n)} $ as defined in \Cref{eq:soft-omp-gap}.
	\end{lemma}
	\begin{proof}
		We start by noting that
		\begin{align*}
			\left\|x^{(n)} - \tilde{x}^{(n)}\right\| &= \left\|\left(\Pi^{(n)}\right)^\top w^{(n)} - \left(\tilde{\Pi}^{(n)}\right)^\top\tilde{w}^{(n)}\right\| \\
			&= \left\|\left(\Pi^{(n)}\right)^\top w^{(n)} - \left(\tilde{\Pi}^{(n)}\right)^\top\left(\tilde{w}^{(n)} - w^{(n)} + w^{(n)}\right)\right\| \\
			&\leq \left\|\Pi^{(n)} - \tilde{\Pi}^{(n)}\right\|\left\|w^{(n)}\right\| + \left\|\tilde{\Pi}^{(n)}\right\|\left\|w^{(n)} - \tilde{w}^{(n)}\right\|. \numberthis \label{sOMP:ineq.1}
		\end{align*}
		The sensitivity analysis result in \Cref{th:Trefethen} implies, for the least-squares step in OMP, that
		\[\sup_{\delta B^{(n)}}\left(\frac{\left\|\delta w^{(n)}\right\|}{\left\|w^{(n)}\right\|}\frac{\left\|B^{(n)}\right\|}{\left\|\delta B^{(n)}\right\|}\right) \leq \kappa\left(B^{(n)}\right) + \frac{\kappa\left(B^{(n)}\right)^2\tan \theta^{(n)}}{\eta^{(n)}}.\]
		Taking the deviation with respect to the equivalent variables in Soft-OMP, rearranging the inequality and substituting $ \theta^{(n)} $ and $ \eta^{(n)} $ as in \Cref{th:Trefethen}, we have
		\begin{align*}
			\left\|w^{(n)}- \tilde{w}^{(n)}\right\| &\leq \left(\kappa\left(B^{(n)}\right) + \frac{\kappa\left(B^{(n)}\right)^2\tan \theta^{(n)}}{\eta^{(n)}}\right)\frac{\left\|\delta B^{(n)}\right\|}{\left\|B^{(n)}\right\|}\left\|w^{(n)}\right\| \\
			&= \left(\kappa\left(B^{(n)}\right) + \kappa\left(B^{(n)}\right)^2\frac{\left\|r^{(n)}\right\|}{\left\|B^{(n)}\right\|\left\|w^{(n)}\right\|}\right)\frac{\left\|\delta B^{(n)}\right\|}{\left\|B^{(n)}\right\|}\left\|w^{(n)}\right\| \numberthis \label{sOMP:ineq.2},
		\end{align*}
		where $\delta B$ is the perturbation introduced to the matrix $B^{(n)} = A(\Pi^{(n)})^\top$ due to approximation of $\Pi^{(n)}$ by $\tilde{\Pi}^{(n)}$. We further note that
		\[\tan(\theta^{(n)}) = \frac{\sin\left(\cos^{-1}\left(\|Ax^{(n)}\|/\|y\|\right)\right)}{\cos\left(\cos^{-1}\left(\|Ax^{(n)}\|/\|y\|\right)\right)} = \frac{\sqrt{\|y\|^2 - \|Ax^{(n)}\|^2}/\|y\|}{\|Ax^{(n)}\|/\|y\|} = \frac{\|r^{(n)}\|}{\|B^{(n)}w^{(n)}\|},\]
		where we defined $ r^{(n)} := y - Ax^{(n)} = y - B^{(n)}w^{(n)} $ and used orthogonality of $r^{(n)}$ to $Ax^{(n)}$ thanks to least squares. Substituting \Cref{sOMP:ineq.2} in \Cref{sOMP:ineq.1} and rearranging, we get
		\begin{align*}
			\left\|x^{(n)} - \tilde{x}^{(n)}\right\| &\leq \left\|\Pi^{(n)} - \tilde{\Pi}^{(n)}\right\|\left\|w^{(n)}\right\| + \left\|\tilde{\Pi}^{(n)}\right\|\left\|w^{(n)} - \tilde{w}_{n}\right\| \\
			&\leq \left(\left\|\Pi^{(n)} - \tilde{\Pi}^{(n)}\right\| + \kappa\left(B^{(n)}\right)\left\|\tilde{\Pi}^{(n)}\right\|\frac{\left\|A\left(\Pi^{(n)} - \tilde{\Pi}^{(n)}\right)^\top\right\|}{\left\|B^{(n)}\right\|}\right)\left\|w^{(n)}\right\| \\
			&\qquad\qquad + \kappa\left(B^{(n)}\right)^2\left\|\tilde{\Pi}^{(n)}\right\|\frac{\left\|A\left(\Pi^{(n)} - \tilde{\Pi}^{(n)}\right)^\top\right\|}{\left\|B^{(n)}\right\|^2}\left\|r^{(n)}\right\|.
		\end{align*}
		Taking advantage of the relations
		\begin{align*}
			&\kappa\left(B^{(n)}\right) = \left\|B^{(n)}\right\|\left\|{B^{(n)}}^\dagger\right\|,\quad w^{(n)} = {B^{(n)}}^\dagger y\\
			&\left\|\Pi^{(n)} - \tilde{\Pi}^{(n)}\right\| \leq \left\|\Pi^{(n)} - \tilde{\Pi}^{(n)}\right\|_F,\qquad \left\|\tilde{\Pi}^{(n)}\right\| \leq \left\|\tilde{\Pi}^{(n)}\right\|_F,
		\end{align*}
		and knowing that $ \|r^{(n)}\| \leq \|y\| $ (see, \cite{davis1994adaptive,foucart2013mathematical}) we have
		\begin{align*}
			\left\|x^{(n)} - \tilde{x}^{(n)}\right\| \leq \left(\left\|\Pi^{(n)} - \tilde{\Pi}^{(n)}\right\| + \left\|{B^{(n)}}^\dagger\right\|\left\|A\right\|\left\|\tilde{\Pi}^{(n)}\right\|\left\|\Pi^{(n)} - \tilde{\Pi}^{(n)}\right\|\right)\left\|{B^{(n)}}^\dagger\right\|\left\|y\right\|& \\
			+ \left\|{B^{(n)}}^\dagger\right\|^2\left\|A\right\|\left\|\tilde{\Pi}^{(n)}\right\|\left\|\Pi^{(n)} - \tilde{\Pi}^{(n)}\right\|\left\|y\right\|&.
		\end{align*}
		We further note that
		\[\left\|{{B^{(n)}}}^\dagger\right\| = \left\|\left(A{\Pi^{(n)}}^\top\right)^\dagger\right\| = \left\|{A^\dagger_{S^{(n)}}}\right\| = \frac{1}{\sigma_{\min}(A_{S^{(n)}})} \leq \frac{1}{\sqrt{1 - \delta_{n}(A)}},\]
		where the last inequality is well-known in the compressive sensing literature (see, e.g., \cite[Ch.~6]{foucart2013mathematical}). Therefore, applying the upper bounds of \Cref{lemma:frob_bounds} and with some rearrangement we conclude that
		\begin{align*}
			\left\|x^{(n)} - \tilde{x}^{(n)}\right\| &\leq \sqrt{2n}(N - 1)\left(\frac{\sqrt{1 - \delta_n(A)} + 2\sqrt{n}\|A\|}{1 - \delta_n(A)}\right)\|y\|e^{-\tilde{g}^{(1:n)}/\tau}.
		\end{align*}
	\end{proof}
	
	\begin{proof}[Proof of \Cref{thm:omp_main}]
		~
		\paragraph*{(i)} Part (i) of the theorem directly follows from the \hyperref[props:asym]{\Cref{prop:properties}\ref{props:asym}}. If $ \tau \to 0 $, Soft-OMP coincides with pOMP and thus with OMP.
		
		\paragraph*{(ii)} We prove Part (ii) by finite induction. The claim of the induction is
			$$
			\argmax{}{(\Pi^{(k)})} = \argmax{}{(\tilde{\Pi}^{(k)})} \quad \text{and} \quad  \max_{i \in [k]}\\\|x^{(i)} - \tilde{x}^{(i)}\| \leq \epsilon. 
			$$
			Since the induction is finite over $k=0,\ldots, n$, this implies, in particular, the claim for $k=n$ as in the theorem's statement.
		\subparagraph*{$ \boldsymbol{(k = 0)} $} This case trivially holds as OMP and Soft-OMP are initialized as $x^{(0)} = \tilde{x}^{(0)} =0$ and $\Pi^{(0)} = \tilde{\Pi}^{(0)}$ are empty matrices.
		
		\subparagraph*{$ \boldsymbol{(k - 1 \rightarrow k)} $} 
		Assume that $ \argmax{}{(\Pi^{(k - 1)})} = \argmax{}{(\tilde{\Pi}^{(k - 1)})} $ and $ \max_{i \in [k-1]}\\\|x^{(i)} - \tilde{x}^{(i)}\| \leq \epsilon $. We will show that $ \argmax{}{(\Pi^{(k)})} = \argmax{}{(\tilde{\Pi}^{(k)})} $ and $ \max_{i \in [k]}\\\|x^{(i)} - \tilde{x}^{(i)}\| \leq \epsilon $. 
		
		We start by noting that for all $ j \in [N] $,
		\begin{flalign*}\label{eq:omp_cross_gap}
			\left|v_j^{(k)} - \tilde{v}_j^{(k)}\right| &= \left|\left(A^*A\left(x^{(k - 1)}\right)\right)_j - \left(A^*A\left(\tilde{x}^{(k - 1)}\right)\right)_j\right| \\
			&= \left|\left\langle a_j, A\left(x^{(k - 1)} - \tilde{x}^{(k - 1)}\right)\right\rangle\right| \leq \left\|A\left(x^{(k - 1)} - \tilde{x}^{(k - 1)}\right)\right\| \\
			&\leq \|A\|\left\|x^{(k - 1)} - \tilde{x}^{(k - 1)}\right\| \leq \|A\|\epsilon.
		\end{flalign*}
		For $ j^{(k)} = \argmax{j \in [N]}{(v_j^{(k)})} $, it can be easily verified that the above inequality, combined with the assumption $ g^{(k)} > 2\|A\|\epsilon $, implies
		\[\left|v_j^{(k)} - \tilde{v}_j^{(k)}\right| < g^{(k)}/2,\]
		for all $j \in [N]$, implying $ \argmax{}{(v^{(k)})} = \argmax{}{(\tilde{v}^{(k)})} $ and as a result $ \argmax{}{(\Pi^{(k)})} = \argmax{}{(\tilde{\Pi}^{(k)})} $. On the other hand, we notice that
		\begin{align*}
			\left|\tilde{v}_i^{(k)} - \tilde{v}_j^{(k)}\right| &= \left|\tilde{v}_i^{(k)} - v_i^{(k)} + v_i^{(k)} - \tilde{v}_j^{(k)} + v_j^{(k)} - v_j^{(k)}\right| \\
			&\geq \left|v_i^{(k)} - v_j^{(k)}\right| - \left|v_i^{(k)} - \tilde{v}_i^{(k)}\right| - \left|v_j^{(k)} - \tilde{v}_j^{(k)}\right| \geq \left|v_i^{(k)} - v_j^{(k)}\right| - 2\|A\|\epsilon, \label{eq:gap_inequality}\numberthis
		\end{align*}
		which implies $\tilde{g}^{(k)} \geq g^{(k)} - 2\|A\|\epsilon > 0$ (recall \Cref{eq:soft-omp-gap}). Now we are in a position to apply \Cref{lemma:distance_bound}, which yields
		\[\left\|x^{(k)} - \tilde{x}^{(k)}\right\| \leq C^{(k)}e^{-\tilde{g}^{(1:k)}/\tau} \leq C^{(k)}e^{\left(-g^{(1:k)} + 2\|A\|\epsilon\right)/\tau},\]
		with $ \tilde{g}^{(1:k)} $ and $g^{(1:k)}$ as defined in \Cref{eq:soft-omp-gap} and \Cref{eq:omp-gap}, respectively. Applying the bound \eqref{eq:tau-omp} on $\tau$ concludes the proof.
	\end{proof}
	
	\section{Proof of {\Cref{thm:iht_main}}}\label{appendix:iht}
	The key component of the proof of \Cref{thm:iht_main} is a bound on the distance between IHT and Soft-IHT over two consecutive iterations, as established in the following lemma. We later use this relation in an inductive argument to derive the main result.
	\begin{lemma}\label{lemma:iht_single_iteration}
		For any $ s $-sparse $ x \in \bC^N $, let $ u = u(x) := (I - \eta A^*A)x + \eta A^*y $, where $ A \in \bC^{m \times N} $ with $ \ell^2 $-normalized columns and $ y \in \bC^m $, and define $ v = |u| $. Also, let $ x^+ $ be the signal after one iteration of IHT, i.e., $ x^+ = H_s(u) = q\odot u $ where
		\[q = \sum_{i = 1}^{s}Q(i, :), \quad Q = P_{\argsort(v)}.\]
		Correspondingly, for any $ \tilde{x} \in \bC^N $ with $ \|x - \tilde{x}\| \leq g/2L $, where
		\[g:= \min_{i, j \in [N], i \neq j}|v_i - v_j|,\]
		and $ L := \|I - \eta A^*A\| $ is the Lipschitz constant of the operator $ u(\cdot) $, define $ \tilde{u} = u(\tilde{x}) $ and $ \tilde{v} = |\tilde{u}| $, and let $ \tilde{x}^+ $ be the signal after one iteration of Soft-IHT, i.e., $ \tilde{x}^+ = \tilde{H}_s(\tilde{u}) = \tilde{q}\odot \tilde{u} $ with
		\[\tilde{q} = \sum_{i = 1}^s\tilde{Q}(i, :),\quad \tilde{Q} = \softsort{(\tilde{v})}.\]
		Then, we have
		\[\|x^+ - \tilde{x}^+\| \leq 2sN\left((|1-\eta| + \eta s\mu)\|x\|_\infty + \eta\|y\|\right) e^{-\tilde{g}/\tau} + sL\|x - \tilde{x}\|,\]
		where $ \mu $ is the coherence of the matrix $ A $ and $ \tilde{g} $ is the minimum gap between elements of $ \tilde{v} $, i.e.,
		\[\tilde{g} := \min_{i, j \in [N], i \neq j}|\tilde{v}_i - \tilde{v}_j|.\]
	\end{lemma}
	\begin{proof}
		The proof is broken into three parts:
		
		\paragraph{Part I} The submultiplicative inequality implies that $ u $ is a Lipschitz operator, i.e., $ \|u - \tilde{u}\| \leq L\|x - \tilde{x}\| $, where $ L := \|I - \eta A^*A\| $.
		
		\paragraph{Part II} First note that, for each $ i \in [N] $,
		\begin{equation}\label{eq:cross_gap}
			|v_i - \tilde{v}_i| = \left||u_i| - |\tilde{u}_i|\right| \leq |u_i - \tilde{u}_i| \leq \|u - \tilde{u}\|_\infty \leq \|u - \tilde{u}\|.
		\end{equation}
		Let
		\[g:= \min_{i, j \in [N], i \neq j}|v_i - v_j|.\]
		Using Part I, we observe that under the condition $ g > 2L\|x - \tilde{x}\| $, it holds
		\begin{equation}\label{eq:gap_cross_gap}
			|v_i - \tilde{v}_i| < g/2,
		\end{equation}
		implying that $ \argsort{(v)} = \argsort{(\tilde{v})} $.
		
		\paragraph{Part III} We then arrive at bounding $ \|x^+ - \tilde{x}^+\| $. We see that
		\begin{align*}
			\|x^+ - \tilde{x}^+\| &= \|H_s(u) - \tilde{H}_s(\tilde{u})\| = \|q\odot u - \tilde{q}\odot \tilde{u}\| \\
			&= \|q\odot u - \tilde{q}\odot (\tilde{u} - u + u)\| \leq \|(q - \tilde{q})\odot u\| + \|\tilde{q}\odot (u - \tilde{u})\| \\
			&\leq \|(q - \tilde{q})\odot u\|_1 + \|\tilde{q}\odot (u - \tilde{u})\|_1 \leq \underbrace{\|q - \tilde{q}\|_1}_{(a)}\underbrace{\|u\|_\infty}_{(b)} + \underbrace{\|\tilde{q}\|_1}_{(c)}\underbrace{\|u - \tilde{u}\|_\infty}_{(d)},
		\end{align*}
		where in the last inequality we used $ \|a\odot b\|_1 = |\dotprod{a}{b}| \leq \|a\|_1\|b\|_\infty $ by H\"{o}lder's inequality. Now we bound each term separately:
		\begin{enumerate}[leftmargin=2.3em, label=(\alph*)]
			\item  Knowing that $ \argsort{}{(v)} = \argsort{}{(\tilde{v})} $ as a consequence of Part II, \hyperref[props:perm_eq]{\Cref{prop:properties}\ref{props:perm_eq}} of softsort implies
			\begin{align*}
				\softsort{(\tilde{v})} &= \softsort{(\sort{(\tilde{v})})}P_{\argsort{(\tilde{v})}} \\
				&= \softsort{(\sort{(\tilde{v})})}P_{\argsort{(v)}} = \softsort{(\sort{(\tilde{v})})}Q.
			\end{align*}
			As the $ \ell^1 $-norm is permutation invariant, a similar argument to \Cref{lemma:frob_bounds} shows
			\allowdisplaybreaks
			\begin{align*}
				\|q - \tilde{q}\|_1 &= \left\|\sum_{i = 1}^sQ(i, :) - \sum_{i = 1}^s\softsort{\left(\sort{(\tilde{v})}\right)}Q(i, :)\right\|_1 \\
				&= \left\|\sum_{i = 1}^s\left(I(i, :) - \softsort{\left(\sort{(\tilde{v})}\right)}(i, :)\right)Q\right\|_1 \\
				&= \left\|\sum_{i = 1}^s\left(I(i, :) - \softsort{\left(\sort{(\tilde{v})}\right)}(i, :)\right)\right\|_1 \\
				&= \left\|\sum_{i = 1}^s\left(e_i - \softmax{\left(-\frac{|\sort{(\tilde{v})}_i\mathbbm{1}^\top - \sort{(\tilde{v})}^\top|}{\tau}\right)}\right)\right\|_1 \\
				&\leq \sum_{i = 1}^{s}\left(1 - \frac{1}{\alpha_i}\right) + \sum_{j = 1, j \neq i}^N\sum_{i = 1}^{s}\frac{e^{-|\tilde{v}_i - \tilde{v}_j|/\tau}}{\alpha_i} \leq \sum_{i = 1}^{s}2\left(1 - \frac{1}{\alpha_i}\right) \\
				&= 2\sum_{i = 1}^{s}\frac{\alpha_i - 1}{\alpha_i} \leq 2\sum_{i = 1}^{s}(\alpha_i - 1) \leq 2\sum_{i = 1}^{s}\sum_{j = 1, j \neq i}^Ne^{-|\tilde{v}_i - \tilde{v}_j|/\tau} \leq 2sNe^{-\tilde{g}/\tau},
			\end{align*}
			where $\alpha_i := \sum_{j = 1}^{N}\exp(-\left|\tilde{v}_{i} - \tilde{v}_j\right|/\tau) $.
			\item Noting that $ \|u\|_\infty \leq \|(I - \eta A^*A)x\|_\infty + \eta\|A^*y\|_\infty $. For the first term, we write
			\begin{align*}
				\|(I - \eta A^*A)x\|_\infty &= \max_{i \in [N]}\left|\left((I - \eta A^*A)x\right)_i\right| = \max_{i \in [N]}\left|\left\langle e_i - \eta a_i^* A, x\right\rangle\right| \\ 
				&= \max_{i \in [N]}\left|x_i - \eta\left\langle a_i^* A, x\right\rangle\right| = \max_{i \in [N]}\left|\left(1 - \eta\|a_i\|^2\right)x_i - \eta\sum_{j = 1, j \neq i}^N\left\langle a_i, a_j\right\rangle x_j\right| \\
                &\leq \left|1 - \eta\right|\max_{i \in [N]}|x_i| + \max_{i \in [N]}\eta\sum_{j \in S, j \neq i}\left|\left\langle a_i, a_j\right\rangle\right||x_i|\leq \left(|1-\eta| + \eta s\mu\right)\|x\|_\infty,
			\end{align*}
			where $ S = \supp(x) $. Furthermore, $\|A^*y\|_\infty = \max_{i \in [N]}|\dotprod{a_i}{y}| \leq \|y\|$. Therefore,
			\begin{align*}
			    \|u\|_\infty &\leq \|(I - \eta A^*A)x\|_\infty + \eta\|A^*y\|_\infty\\
                &\leq \left(|1-\eta| + \eta s\mu\right)\|x\|_\infty + \eta\max_{i \in [N]}|\dotprod{a_i}{y}| = (|1-\eta| + \eta s\mu)\|x\|_\infty + \eta\|y\|.
			\end{align*}
			\item As we sum over $ s $ rows of the softsort matrix and by switching the sums over rows and columns, thanks to \hyperref[props:row_affine]{\Cref{prop:properties}\ref{props:row_affine}} of softsort we have $ \|\tilde{q}\|_1 = \sum_{i = 1}^s\|\tilde{Q}(i, :)\|_1 = s $.
			\item We derived earlier that $ \|u - \tilde{u}\|_\infty \leq \|u - \tilde{u}\| \leq L\|x - \tilde{x}\| $.
		\end{enumerate}
		~\\
		Putting all bounds of (a)--(d) together concludes the proof.
	\end{proof}
	Now we are in a position to set out the proof of the main IHT theorem.
	\begin{proof}[Proof of {\Cref{thm:iht_main}}]
		~
		\paragraph*{(i)} Part (i) of the theorem directly follows from \hyperref[props:asym]{\Cref{prop:properties}\ref{props:asym}} of the softsort as in \Cref{prop:properties}. If $ \tau \to 0 $, Soft-IHT coincides with pIHT and thus with IHT.
		
		\paragraph*{(ii)} Using the result of \Cref{lemma:iht_single_iteration} and letting $C'^{(n)}:=(|1-\eta| + \eta s\mu)\|x^{(n)}\| + \eta\|y\|$, by induction we can write
        \begin{align*}
			\|x^{(n)} - \tilde{x}^{(n)}\| &\leq 2sNC'^{(n-1)} e^{-\tilde{g}^{(n)}/\tau} + sL\|x^{(n - 1)} - \tilde{x}^{(n - 1)}\| \\
			&\leq 2sNC'^{(n-1)} e^{-\tilde{g}^{(n)}/\tau} + sL\Big(2sNC'^{(n-2)} e^{-\tilde{g}^{(n - 1)}/\tau} + sL\|x^{(n - 2)} - \tilde{x}^{(n - 2)}\|\Big) \\
			&\leq 2sN(1 + sL)\max\left\{C'^{(n-1)}, C'^{(n-2)}\right\} e^{-\min\left\{\tilde{g}^{(n)}, \tilde{g}^{(n - 1)}\right\}/\tau}\\
			&\quadmul{9} + (sL)^2\|x^{(n - 2)} - \tilde{x}^{(n - 2)}\| \\
			&\quad\vdots \\
			&\leq 2sN\frac{(sL)^n - 1}{sL - 1}\max_{0 \leq k \leq n - 1}C'^{(k)}e^{-\tilde{g}^{(1:n)}/\tau}+ (sL)^n\|x^{(0)} - \tilde{x}^{(0)}\|\\
            &\leq 2sN\frac{(sL)^n - 1}{sL - 1}\left((|1-\eta| + \eta s\mu)\max_{0 \leq k \leq n - 1}\|x^{(k)}\|_\infty + \eta\|y\|\right)e^{-\tilde{g}^{(1:n)}/\tau}\\
            &\quadmul{12} + (sL)^n\|x^{(0)} - \tilde{x}^{(0)}\|\numberthis\label{eq:iht_soft_iht_bd_1}. 
		\end{align*}
		Here $\tilde{g}^{(1:n)}$ is defined as
		\[\tilde{g}^{(1:n)} := \min_{k \in [n]}\tilde{g}^{(k)},\quad \tilde{g}^{(k)} = \min_{i, j \in [N], i \neq j}\left|v_i^{(k)} - v_j^{(k)}\right|,\]
		and the result holds under the condition $\max_{1 \leq k \leq n - 1}\|x^{(k)} - \tilde{x}^{(k)}\| \leq g^{(1:n)}/2L$ with $g^{(1:n)}$ as in \Cref{eq:iht_gap}. Under this condition, a similar analysis as in the proof of \Cref{thm:omp_main} (see \Cref{appendix:omp}) leads to $\tilde{g}^{(1:n)} \geq g^{(1:n)} - 2L\max_{1 \leq k \leq n - 1}\|x^{(k)} - \tilde{x}^{(k)}\| > 0$, connecting the global min gap of Soft-IHT to the one of IHT. Hence, the upper bound \Cref{eq:iht_soft_iht_bd_1} can be written as
		\begin{equation}
			\label{eq:iht_soft_iht_bd_2}
			\|x^{(n)} - \tilde{x}^{(n)}\| \leq C^{(n)}e^{\left(-g^{(1:n)} + 2L\max_{1 \leq k \leq n - 1}\|x^{(k)} - \tilde{x}^{(k)}\|\right)/\tau} + (sL)^n\|x^{(0)} - \tilde{x}^{(0)}\|,
		\end{equation}
		where
		\[C^{(n)} = 2sN\frac{(sL)^n - 1}{sL - 1}\left(((|1-\eta| + \eta s\mu)\max_{0 \leq k \leq n - 1}\|x^{(k)}\|_\infty + \eta\|y\|\right).\]
		Now using the fact that $ x^{(0)} = \tilde{x}^{(0)} $, applying on the bound \eqref{eq:iht_soft_iht_bd_2}, the assumption \eqref{eq:tau-iht} for $\tau$, concludes the proof.
	\end{proof}
	
	\section{Impact of the measurement matrix on the temperature parameter} \label{sec:impact_A_on_tau}
	To conclude, we investigate the impact of the measurement matrix $A$ on the convergence of the Soft version of OMP and IHT to their standard conterparts with respect to the temperature parameter $\tau$. We run a numerical experiment obtained by considering the same setting of Experiment I (see \Cref{fig:shaded_plot}) and replacing random Gaussian measurements with random Fourier measurements. The results are shown in \Cref{fig:shaded_plot_Fourier} and are almost identical to those in \Cref{fig:shaded_plot}. 
		\begin{figure}[t!]
			\includegraphics[width = 0.49\textwidth]{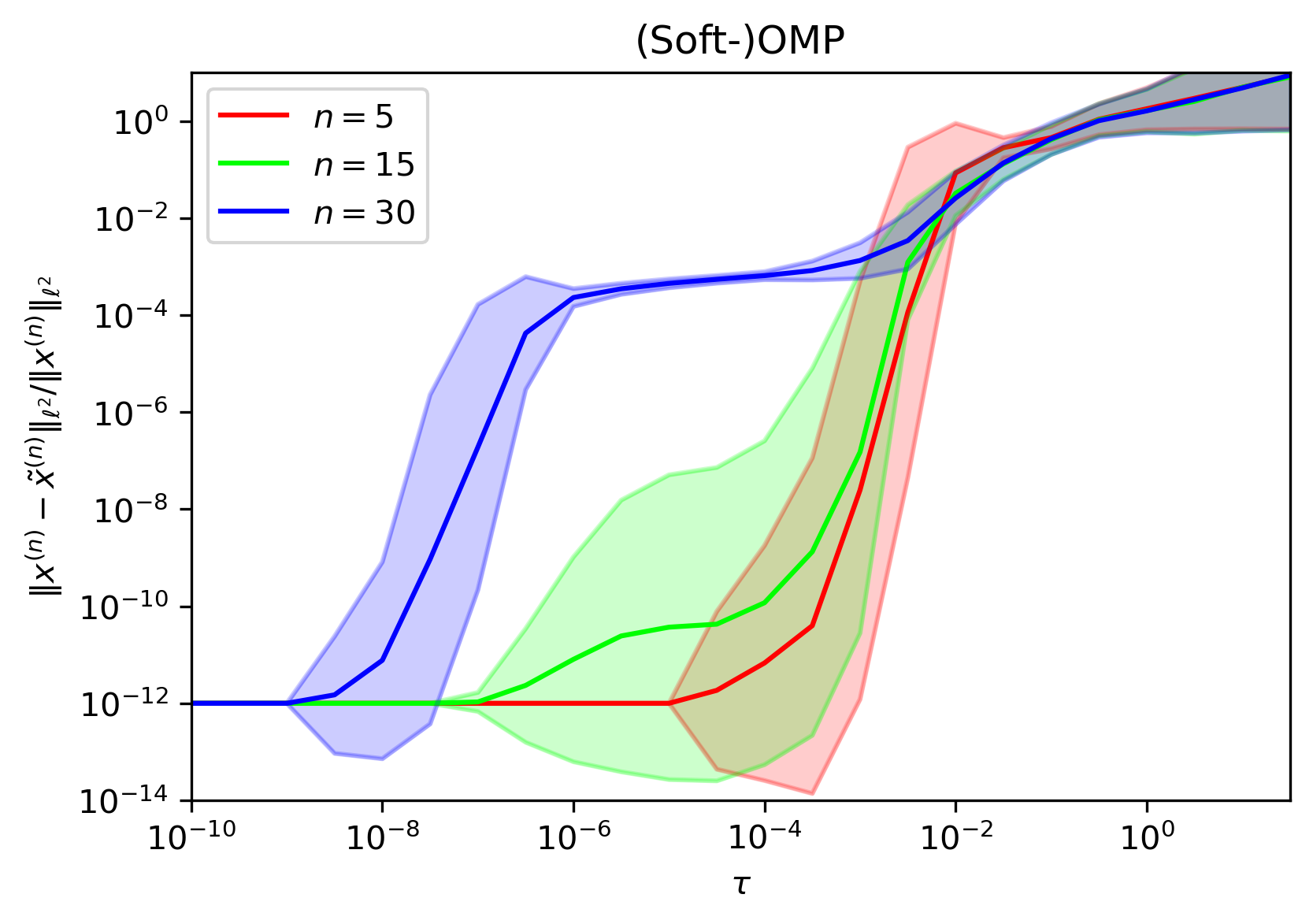}%
			\includegraphics[width = 0.49\textwidth]{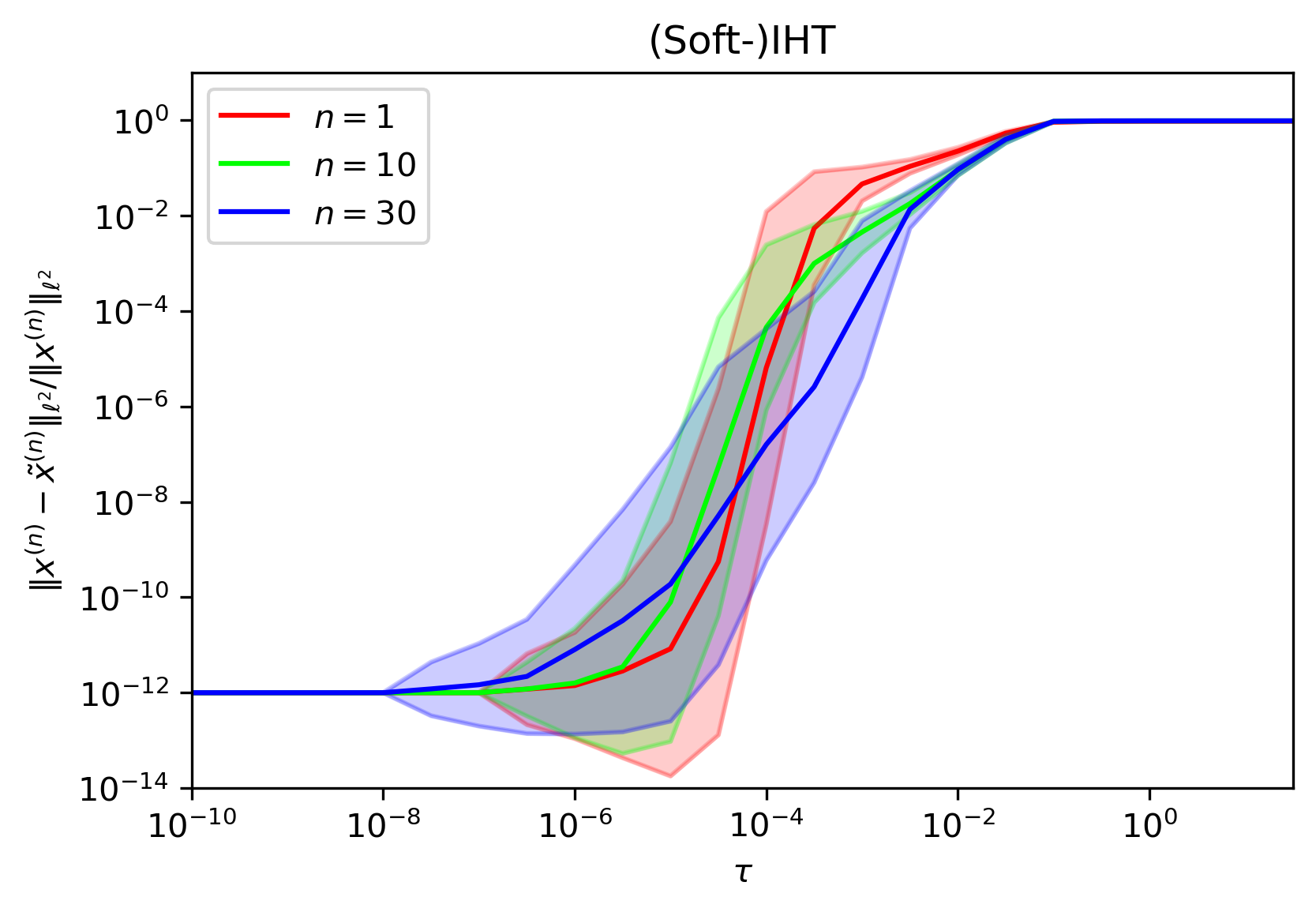}
			\caption{Relative $\ell^2$-error as a function of $\tau$ (see \Cref{subsec:experiment_i}). Recovery accuracy of (Soft-)OMP on the left and (Soft-)IHT on the right for various iteration counts. The only difference with \Cref{fig:shaded_plot} is that Gaussian measurements are replaced with Fourier measurements.}
			\label{fig:shaded_plot_Fourier}
		\end{figure}
		This suggests that the impact of the type of measurements on the choice of $\tau$---which drives the convergence of Soft-OMP and Soft-IHT to OMP and IHT, repsectively---is negligible in this setting.
	
	\bibliographystyle{siamplain}
	\bibliography{myref.bib}
\end{document}

%% file: ex_shared.tex
\usepackage{hyperref}
\usepackage{cleveref}
\usepackage{amssymb, amsmath}
\usepackage{bbm}
\usepackage{enumitem}
\usepackage{dsfont}
\usepackage{tcolorbox}
\usepackage{ntheorem}
\usepackage{algpseudocode}
\usepackage{comment}
\usepackage{float}
\usepackage{xr}
\usepackage[normalem]{ulem}
\graphicspath{{./images/}}
\newcommand\numberthis{\addtocounter{equation}{1}\tag{\theequation}}

\AddToHook{cmd/appendix/before}{\crefalias{section}{appendix}}

\makeatletter
\AtBeginDocument{
	\def\refstepcounter@optarg[#1]#2{%
		\cref@old@refstepcounter{#2}%
		\cref@constructprefix{#2}{\cref@result}%
		\@ifundefined{cref@#1@alias}%
		{\def\@tempa{#1}}%
		{\def\@tempa{\csname cref@#1@alias\endcsname}}%
		\protected@edef\cref@currentlabel{%
			[\@tempa][\arabic{#2}][\cref@result]%
			\csname p@#2\endcsname\csname the#2\endcsname}}}      
\makeatother

\AfterEndEnvironment{equation}{\restarthintedrel}
\AfterEndEnvironment{align}{\restarthintedrel}
\newcounter{hints}
\renewcommand{\thehints}{\roman{hints}}
\newcommand{\hintedrel}[2][]{%
	\stepcounter{hints}%
	\if\relax\detokenize{#1}\relax\else\csxdef{hint@#1}{\thehints}\fi
	\mathrel{\overset{\textrm{(\thehints)}}{\vphantom{\le}{#2}}}%
}
\newcommand{\restarthintedrel}{\setcounter{hints}{0}}

\def\quadmul#1{%
	\foreach \index in {1, ..., #1}{%
		\qquad
	}}

\newcommand{\graysout}[1]{\textcolor{gray}{\sout{#1}}}
\definecolor{my-back-col-red}{rgb}{1 0.95 0.95}
\definecolor{my-tit-col-red}{rgb}{0.5 0 0}
\definecolor{my-back-col-blue}{rgb}{0.90 0.95 1}
\definecolor{my-tit-col-blue}{rgb}{0 0.26 0.5}
\definecolor{my-back-col-yellow}{rgb}{1 1 0.8}

\newcommand{\bR}{\mathbb{R}}
\newcommand{\bC}{\mathbb{C}}
\newcommand{\bN}{\mathbb{N}}
\newcommand{\dotprod}[2]{\langle #1, #2 \rangle}
\DeclareMathOperator{\supp}{supp}
\DeclareMathOperator{\card}{card}
\DeclareMathOperator{\sort}{sort}
\DeclareMathOperator{\argsort}{argsort}
\DeclareMathOperator{\softmax}{softmax}
\DeclareMathOperator{\softsort}{softsort_\tau}
\newcommand{\argmax}[2]{\mathop{\mathrm{argmax}}_{#1}#2}
\newcommand{\argmin}[2]{\mathop{\mathrm{argmin}}_{#1}#2}

\newtheorem{remark}[theorem]{Remark}
\newtheorem{example}[theorem]{Example}

\AddToHook{cmd/remark/before}{\crefalias{theorem}{remark}}
\AddToHook{cmd/example/before}{\crefalias{theorem}{example}}
\AddToHook{cmd/assumption/before}{\crefalias{theorem}{assumption}}
\AddToHook{cmd/lemma/before}{\crefalias{theorem}{lemma}}
\AddToHook{cmd/proposition/before}{\crefalias{theorem}{proposition}}

\crefalias{property}{enumi}
\crefname{property}{property}{properties}
\Crefname{property}{Property}{Properties}

\title{Deep greedy unfolding: Sorting out argsorting in greedy sparse recovery algorithms\footnote{This work is an extension of \cite{mohammad2024omp}.}}
\author{Sina Mohammad-Taheri\footnote{Department of Mathematics and Statistics, Concordia University, Montr\'eal, QC, Canada.},$\ \ $Matthew J. Colbrook$^{}$\footnote{Department of Applied Mathematics and Theoretical Physics, University of Cambridge, Cambridge, Cambridgeshire, UK},$\ \ $and$\ \ $Simone Brugiapaglia$^\dag$}

%% file: myref.bib
@inproceedings{li2013weighted,
  title={A weighted {OMP} algorithm for compressive {UWB} channel estimation},
  author={Li, Guo Zhu and Wang, De Qiang and Zhang, Zi Kai and Li, Zhi Yong},
  booktitle={Applied Mechanics and Materials},
  volume={392},
  pages={852--856},
  year={2013}
}

@article{friedlander2011recovering,
  title={Recovering compressively sampled signals using partial support information},
  author={Friedlander, Michael P and Mansour, Hassan and Saab, Rayan and Yilmaz, {\"O}zg{\"u}r},
  journal={IEEE Transactions on Information Theory},
  volume={58},
  number={2},
  pages={1122--1134},
  year={2011},
  publisher={IEEE}
}

@Book{foucart2013mathematical,
  Title                    = {A Mathematical Introduction to Compressive Sensing},
  Author                   = {Foucart, Simon and Rauhut, Holger},
  Publisher                = {Birkhäuser},
  address = {New York, NY},
  Year                     = {2013}
}

@Article{jo2013iterative,
  Title                    = {Iterative hard thresholding for weighted sparse approximation},
  Author                   = {Jo, Jason},
  Journal                  = {arXiv preprint arXiv:1312.3582},
  Year                     = {2013}
}

@Article{rauhut2016interpolation,
  Title                    = {Interpolation via weighted $\ell_1$ minimization},
  Author                   = {Rauhut, Holger and Ward, Rachel},
  Journal                  = {Applied and Computational Harmonic Analysis},
  Year                     = {2016},
  Number                   = {2},
  Pages                    = {321--351},
  Volume                   = {40},

  Publisher                = {Elsevier}
}

@inproceedings{gregor2010learning,
  title={Learning fast approximations of sparse coding},
  author={Gregor, Karol and LeCun, Yann},
  booktitle={Proceedings of the 27th International Conference on Machine Learning},
  pages={399--406},
  year={2010}
}

@book{adcock2022sparse,
author = {Adcock, Ben and Brugiapaglia, Simone and Webster, Clayton G.},
title = {Sparse Polynomial Approximation of High-Dimensional Functions},
publisher = {Society for Industrial and Applied Mathematics},
address = {Philadelphia, PA},
year = {2022},
edition   = {},
}

@article{zhang2011sparse,
  title={Sparse recovery with orthogonal matching pursuit under {RIP}},
  author={Zhang, Tong},
  journal={IEEE Transactions on Information Theory},
  volume={57},
  number={9},
  pages={6215--6221},
  year={2011},
  publisher={IEEE}
}

@article{khajehnejad2011analyzing,
  title={Analyzing weighted $\ell_1 $ minimization for sparse recovery with nonuniform sparse models},
  author={Khajehnejad, M. Amin and Xu, Weiyu and Avestimehr, A. Salman and Hassibi, Babak},
  journal={IEEE Transactions on Signal Processing},
  volume={59},
  number={5},
  pages={1985--2001},
  year={2011},
  publisher={IEEE}
}

@article{bah2016sample,
  title={The sample complexity of weighted sparse approximation},
  author={Bah, Bubacarr and Ward, Rachel},
  journal={IEEE Transactions on Signal Processing},
  volume={64},
  number={12},
  pages={3145--3155},
  year={2016},
  publisher={IEEE}
}

@article{davis1994adaptive,
  title={Adaptive time-frequency decompositions},
  author={Davis, Geoffrey M and Mallat, Stephane G and Zhang, Zhifeng},
  journal={Optical Engineering},
  volume={33},
  number={7},
  pages={2183--2191},
  year={1994},
  publisher={SPIE}
}

@inproceedings{pati1993orthogonal,
  title={Orthogonal matching pursuit: Recursive function approximation with applications to wavelet decomposition},
  author={Pati, Yagyensh Chandra and Rezaiifar, Ramin and Krishnaprasad, Perinkulam Sambamurthy},
  booktitle={Proceedings of 27th Asilomar Conference on Signals, Systems and Computers},
  pages={40--44},
  year={1993},
  organization={IEEE}
}

@article{blumensath2008iterative,
  title={Iterative thresholding for sparse approximations},
  author={Blumensath, Thomas and Davies, Mike E},
  journal={Journal of Fourier Analysis and Applications},
  volume={14},
  pages={629--654},
  year={2008},
  publisher={Springer}
}

@article{needell2009cosamp,
  title={{CoSaMP}: Iterative signal recovery from incomplete and inaccurate samples},
  author={Needell, Deanna and Tropp, Joel A.},
  journal={Applied and Computational Harmonic Analysis},
  volume={26},
  number={3},
  pages={301--321},
  year={2009},
  publisher={Elsevier}
}

@Article{xin2016maximal,
  author  = {Xin, Bo and Wang, Yizhou and Gao, Wen and Wipf, David and Wang, Baoyuan},
  journal = {Advances in Neural Information Processing Systems},
  title   = {Maximal sparsity with deep networks?},
  year    = {2016},
  volume  = {29},
}

@Article{colbrook2022difficulty,
  author    = {Colbrook, Matthew J and Antun, Vegard and Hansen, Anders C.},
  journal   = {Proceedings of the National Academy of Sciences},
  title     = {The difficulty of computing stable and accurate neural networks: On the barriers of deep learning and {S}male’s 18th problem},
  year      = {2022},
  number    = {12},
  pages     = {e2107151119},
  volume    = {119},
  publisher = {National Acad Sciences},
}

@Article{monga2021algorithm,
  author    = {Monga, Vishal and Li, Yuelong and Eldar, Yonina C.},
  journal   = {IEEE Signal Processing Magazine},
  title     = {Algorithm unrolling: Interpretable, efficient deep learning for signal and image processing},
  year      = {2021},
  number    = {2},
  pages     = {18--44},
  volume    = {38},
  publisher = {IEEE},
}

@Article{liang2020deep,
  author    = {Liang, Dong and Cheng, Jing and Ke, Ziwen and Ying, Leslie},
  journal   = {IEEE Signal Processing Magazine},
  title     = {Deep magnetic resonance image reconstruction: Inverse problems meet neural networks},
  year      = {2020},
  number    = {1},
  pages     = {141--151},
  volume    = {37},
  publisher = {IEEE},
}

@InProceedings{zhang2018ista,
  author    = {Zhang, Jian and Ghanem, Bernard},
  booktitle = {Proceedings of the IEEE Conference on Computer Vision and Pattern Recognition},
  title     = {{ISTA-Net}: Interpretable optimization-inspired deep network for image compressive sensing},
  year      = {2018},
  pages     = {1828--1837},
}

@Book{adcock2021compressive,
  author    = {Adcock, Ben and Hansen, Anders C.},
  publisher = {Cambridge University Press},
  title     = {Compressive Imaging: Structure, Sampling, Learning},
  year      = {2021},
  address = {Cambridge, UK}
}

@article{mohammad2025greedy,
  title={The greedy side of the LASSO: New algorithms for weighted sparse recovery via loss function-based orthogonal matching pursuit},
  author={Mohammad-Taheri, Sina and Brugiapaglia, Simone},
  journal={Sampling Theory, Signal Processing, and Data Analysis},
  volume={23},
  number={1},
  pages={3},
  year={2025},
  publisher={Springer}
}

@inproceedings{mohammad2024omp,
  title={{OMP-Net}: {N}eural network unrolling of weighted {Orthogonal Matching Pursuit}},
  author={Mohammad-Taheri, Sina and Colbrook, Matthew J and Brugiapaglia, Simone},
  booktitle={2024 International Workshop on the Theory of Computational Sensing and its Applications to Radar, Multimodal Sensing and Imaging (CoSeRa)},
  pages={61--65},
  year={2024},
  organization={IEEE}
}

@article{bouchot2017multi,
  title={Multi-level Compressed Sensing {Petrov-Galerkin} discretization of high-dimensional parametric {PDEs}},
  author={Bouchot, Jean-Luc and Rauhut, Holger and Schwab, Christoph},
  journal={arXiv preprint arXiv:1701.01671},
  year={2017}
}

@inproceedings{wang2016learning,
  title={Learning deep $\ell^0$ encoders},
  author={Wang, Zhangyang and Ling, Qing and Huang, Thomas},
  booktitle={Proceedings of the AAAI Conference on Artificial Intelligence},
  volume={30},
  issue={1},
  year={2016}
}

@article{khatib2021learned,
  title={Learned greedy method ({LGM}): A novel neural architecture for sparse coding and beyond},
  author={Khatib, Rajaei and Simon, Dror and Elad, Michael},
  journal={Journal of Visual Communication and Image Representation},
  volume={77},
  pages={103095},
  year={2021},
  publisher={Elsevier}
}

@article{scarlett2022theoretical,
  title={Theoretical perspectives on deep learning methods in inverse problems},
  author={Scarlett, Jonathan and Heckel, Reinhard and Rodrigues, Miguel RD and Hand, Paul and Eldar, Yonina C},
  journal={IEEE Journal on Selected Areas in Information Theory},
  volume={3},
  number={3},
  pages={433--453},
  year={2022},
  publisher={IEEE}
}

@inproceedings{chen2021deeppursuit,
  title={Deeppursuit: Uniting classical wisdom and deep {RL} for sparse recovery},
  author={Chen, Ziheng and Zhong, Sichen and Chen, Jianshu and Zhao, Yue},
  booktitle={2021 55th Asilomar conference on signals, systems, and computers},
  pages={1361--1366},
  year={2021},
  organization={IEEE}
}

@article{daubechies2004iterative,
  title={An iterative thresholding algorithm for linear inverse problems with a sparsity constraint},
  author={Daubechies, Ingrid and Defrise, Michel and De Mol, Christine},
  journal={Communications on Pure and Applied Mathematics: A Journal Issued by the Courant Institute of Mathematical Sciences},
  volume={57},
  number={11},
  pages={1413--1457},
  year={2004},
  publisher={Wiley Online Library}
}

@book{arora2009computational,
  title={Computational Complexity: A Modern Approach},
  author={Arora, Sanjeev and Barak, Boaz},
  year={2009},
  publisher={Cambridge University Press}
}

@article{fornasier2008iterative,
  title={Iterative thresholding algorithms},
  author={Fornasier, Massimo and Rauhut, Holger},
  journal={Applied and Computational Harmonic Analysis},
  volume={25},
  number={2},
  pages={187--208},
  year={2008},
  publisher={Elsevier}
}

@article{kutyniok2024mathematics,
  title={The mathematics of reliable artificial intelligence},
  author={Kutyniok, Gitta},
  journal={Collections},
  volume={57},
  number={06},
  year={2024}
}

@inproceedings{prillo2020softsort,
  title={Softsort: A continuous relaxation for the argsort operator},
  author={Prillo, Sebastian and Eisenschlos, Julian},
  booktitle={International Conference on Machine Learning},
  pages={7793--7802},
  year={2020},
  organization={PMLR}
}

@article{grover2019stochastic,
  title={Stochastic optimization of sorting networks via continuous relaxations},
  author={Grover, Aditya and Wang, Eric and Zweig, Aaron and Ermon, Stefano},
  journal={arXiv preprint arXiv:1903.08850},
  year={2019}
}

@inproceedings{blondel2020fast,
  title={Fast differentiable sorting and ranking},
  author={Blondel, Mathieu and Teboul, Olivier and Berthet, Quentin and Djolonga, Josip},
  booktitle={International Conference on Machine Learning},
  pages={950--959},
  year={2020},
  organization={PMLR}
}

@inproceedings{sander2023fast,
  title={Fast, differentiable and sparse top-k: a convex analysis perspective},
  author={Sander, Michael Eli and Puigcerver, Joan and Djolonga, Josip and Peyr{\'e}, Gabriel and Blondel, Mathieu},
  booktitle={International Conference on Machine Learning},
  pages={29919--29936},
  year={2023},
  organization={PMLR}
}

@article{cuturi2019differentiable,
  title={Differentiable ranking and sorting using optimal transport},
  author={Cuturi, Marco and Teboul, Olivier and Vert, Jean-Philippe},
  journal={Advances in Neural Information Processing Systems},
  volume={32},
  year={2019}
}

@article{mena2018learning,
  title={Learning latent permutations with {G}umbel-{S}inkhorn networks},
  author={Mena, Gonzalo and Belanger, David and Linderman, Scott and Snoek, Jasper},
  journal={arXiv preprint arXiv:1802.08665},
  year={2018}
}

@article{dai2009subspace,
  title={Subspace pursuit for compressive sensing signal reconstruction},
  author={Dai, Wei and Milenkovic, Olgica},
  journal={IEEE Transactions on Information Theory},
  volume={55},
  number={5},
  pages={2230--2249},
  year={2009},
  publisher={IEEE}
}

@article{jang2016categorical,
  title={Categorical reparameterization with {G}umbel-softmax},
  author={Jang, Eric and Gu, Shixiang and Poole, Ben},
  journal={arXiv preprint arXiv:1611.01144},
  year={2016}
}

@incollection{behboodi2022compressive,
  title={Compressive sensing and neural networks from a statistical learning perspective},
  author={Behboodi, Arash and Rauhut, Holger and Schnoor, Ekkehard},
  booktitle={Compressed Sensing in Information Processing},
  pages={247--277},
  year={2022},
  publisher={Springer}
}

@article{chambolle2011first,
	title={A first-order primal-dual algorithm for convex problems with applications to imaging},
	author={Chambolle, Antonin and Pock, Thomas},
	journal={Journal of Mathematical Imaging and Vision},
	volume={40},
	pages={120--145},
	year={2011},
	publisher={Springer}
}

@book{goodfellow2016deep,
	title={Deep learning},
	author={Goodfellow, Ian and Bengio, Yoshua and Courville, Aaron},
	year={2016},
	publisher={MIT Press}
}

@article{gao2017properties,
	title={On the properties of the softmax function with application in game theory and reinforcement learning},
	author={Gao, Bolin and Pavel, Lacra},
	journal={arXiv preprint arXiv:1704.00805},
	year={2017}
}

@book{trefethen2022numerical,
	title={Numerical Linear Algebra},
	author={Trefethen, Lloyd N and Bau, David},
	year={2022},
	publisher={SIAM}
}

@article{choi20217,
	title={7 revealing ways {AIs} fail: Neural networks can be disastrously brittle, forgetful, and surprisingly bad at math},
	author={Choi, Charles Q},
	journal={IEEE Spectrum},
	volume={58},
	number={10},
	pages={42--47},
	year={2021},
	publisher={IEEE}
}

@article{heaven2019deep,
	title={Why deep-learning {AIs} are so easy to fool},
	author={Heaven, Douglas and others},
	journal={Nature},
	volume={574},
	number={7777},
	pages={163--166},
	year={2019},
	publisher={Springer Science and Business Media LLC}
}
